\documentclass[11pt]{article}

\usepackage[left=1in,top=1in,right=1in,bottom=1in]{geometry}
\usepackage{amsmath,amssymb,amsthm,amsfonts}
\usepackage[colorlinks,citecolor=blue, linkcolor=blue, urlcolor=magenta]{hyperref}
\usepackage{multicol}
\usepackage{paralist}
\usepackage{times}
\usepackage{enumitem}
\usepackage{epsfig}
\usepackage{graphicx}
\usepackage{pgf}
\usepackage{chemarr}
\usepackage{float}
\usepackage{array, multirow}
\usepackage{cleveref}
\restylefloat{table}

\usepackage[linesnumbered,ruled,commentsnumbered,noline]{algorithm2e}

\usepackage{cases}
\usepackage{caption}
\usepackage{subcaption}

\usepackage{url}
\usepackage{bm}
\usepackage[square, numbers]{natbib}
\usepackage{tcolorbox}

\usepackage[mathscr]{eucal}
\usepackage{color}
\usepackage{macro}

\numberwithin{equation}{section} % Number equations by section

\makeatletter
\def\blfootnote{\gdef\@thefnmark{}\@footnotetext}
\makeatother

\title{Nonparametric learning of stochastic differential equations from sparse and noisy data} 
\author{Arnab Ganguly$^{1}$\thanks{Research of A. Ganguly and J. Zhou is supported in part by NSF DMS-1855788 and NSF DMS-2246815. A. Ganguly is also supported by the Simons Foundation (Travel Support for Mathematicians)}, \ Riten Mitra $^{2}$, \ Jinpu Zhou$^{1}$\footnotemark[1]}

\begin{document}
\maketitle
\footnotetext[1]{Department of Mathematics, Louisiana State University, aganguly@lsu.edu (AG), zjinpu1@lsu.edu (JZ) }
\footnotetext[2]{Department of Bioinformatics and Biostatistics, University of Louisville, 	ritendranath.mitra@louisville.edu}
%\blfootnote{Authors contributed equally}

%\let\thefootnote\relax\footnote{heinz.koeppl@bcs.tu-darmstadt.de}
%\linenumbers
\begin{abstract}
The paper proposes a systematic framework for building data-driven stochastic differential equation (SDE) models from sparse, noisy observations. Unlike traditional parametric approaches, which assume a known functional form for the drift, our goal here is to learn the entire drift function directly from data without strong structural assumptions, making it especially relevant in scientific disciplines where system dynamics are partially understood or highly complex. We cast the estimation problem as minimization of the penalized negative log-likelihood functional over  a reproducing kernel Hilbert space (RKHS). In the sparse observation regime, the presence of unobserved trajectory segments makes the SDE likelihood intractable. To address this, we develop an Expectation-Maximization (EM) algorithm that employs a novel Sequential Monte Carlo (SMC) method to approximate the filtering distribution and generate Monte Carlo estimates of the E-step objective. The M-step then reduces to a penalized empirical risk minimization problem in the RKHS, whose minimizer is given by a finite linear combination of kernel functions via a generalized representer theorem.  To control model complexity across EM iterations, we also develop a hybrid Bayesian variant of the algorithm that uses shrinkage priors to identify significant coefficients in the kernel expansion.  We establish important theoretical convergence results for both the exact and approximate EM sequences. The resulting EM-SMC-RKHS procedure enables accurate estimation of the drift function of stochastic dynamical systems in low-data regimes and is broadly applicable across domains requiring continuous-time modeling under observational constraints. We demonstrate the effectiveness of our method through a series of numerical experiments.



\end{abstract}

\vspace{.1in}

\setcounter{equation}{0}
\renewcommand {\theequation}{\arabic{section}.\arabic{equation}}
\section{Introduction.}\label{sec:intro}
 Stochastic differential equations (SDEs) of the form 
\begin{align}\label{eq:sde0}
	\procX(t) = x_0+\int_0^t \drift(\procX(s))ds+\int_0^t \diffus(\procX(s))dW(s), \quad x_0 \in \R^d,
\end{align}
provide a powerful and flexible framework for modeling systems influenced by both deterministic and random forces. They arise naturally in diverse domains ranging from physics (e.g., Langevin dynamics), to quantitative finance (e.g., stochastic volatility models), to systems biology (e.g., gene regulatory networks and population dynamics). The drift term in an SDE governs the deterministic trend of the system, making its accurate estimation critical for understanding and predicting system behavior.

Traditionally, statistical inference for SDEs has focused on parametric models, where the drift function is assumed to have a known functional form, typically denoted as $\drft(\theta, \cdot)$ for a finite-dimensional parameter $\theta$. Estimation then reduces to recovering $\theta$ from observed data using frequentist or Bayesian methods.  There exists an extensive literature on the theoretical and computational aspects of parametric SDE models and their statistical inference; see, for example, \cite{Yos92, Kes97, Chib01, RoSt01, Saha02, Kut03, GoWi05, BPRF06, FePaRo08, Saha08, Bis08, GoWi08, Iacus08, ChCh11, ArOp11, CsOp13, Li13, BlSo14, SuGa16, WGBS17} for a representative subset. While parametric models offer computational convenience and interpretability, they often rely on strong structural assumptions that may only hold under  strict,  idealized conditions in real-world systems. 

In many scientific applications, a data-driven modeling approach is more appropriate --- one that avoids specifying a priori the form of the drift function and instead learns it directly from data. This nonparametric formulation is particularly relevant when prior knowledge is limited or when the system exhibits rich, unknown structure. For example, in cell signaling pathways or epidemiological models, the form of interactions is often only partially understood, and learning the drift from time-series data can reveal latent mechanistic insights. Similarly, in neuroscience, recovering the drift function from noisy voltage traces can help characterize underlying biophysical processes without pre-assuming a particular form.

Despite its relevance, nonparametric drift estimation for SDEs remains relatively underexplored. Most work in machine learning and classical nonparametric statistics focuses on i.i.d. data or data with simple correlation structures, such as in regression or classification settings, which are significantly more tractable. They are not applicable in our setting, where the data arise from SDEs with complex temporal dependencies. Some early efforts for stochastic dynamical systems rely on histogram-based local averaging around bins of small width (\cite{FPST11}), or use refinements like $k$-nearest neighbors (\cite{HeSt09}) and Nadaraya-Watson estimators (\cite{LaKl09}). These methods typically require dense data near each location $x$ and struggle to generalize beyond low-dimensional, toy systems. Gaussian process-based approaches have also been proposed (\cite{RBO13, Yil18}), but they often rely on linearization or other ad hoc approximations, which may not scale well and in some cases introduce theoretical inconsistencies.

In recent work \citep{GaMiZh23}, we proposed a novel method integrating reproducing kernel Hilbert space (RKHS) theory with a Bayesian framework for learning the drift function $b$ from high-frequency observations of the process $X$.   There, the data-rich setting allowed us to approximate the SDE likelihood using the Euler-Maruyama scheme, and the drift function $\drft$ was estimated by minimizing the negative log-likelihood functional, viewed as a function of $\drft$, over an RKHS $\Hsp_{\knl}$. The RKHS structure facilitates a tractable finite-sum representation of the minimizer via a generalization of the classical representer theorem, thereby converting the infinite-dimensional optimization problem to a finite-dimensional one. However, this setup assumes that trajectories of $X$ are densely sampled, and is not applicable to more realistic scenarios where, due to measurement limitations, cost, or system inaccessibility, the process is only sparsely observed.  

In this paper, we address this substantially more challenging problem of nonparametric drift estimation from sparse and noisy observations. That is, we observe only $\{\obsY(t_m)\}$ at time points $\{t_m\}$ --- a sparse, noisy version of the latent diffusion process $X$, which, for convenience, is modeled as a discretized form of the SDE \eqref{eq:sde0} in our paper (see \eqref{eq:sde-disc-0}). In this data-sparse regime, even for noise-free exact observations, standard Euler-based approximations break down, and the likelihood becomes intractable due to the need to integrate over unobserved segments of the $X$ trajectory between observation times. This leads to a challenging infinite-dimensional optimization problem, which we address using an Expectation-Maximization (EM) algorithm in an RKHS framework. One of the main bottlenecks here is that the E-step in each iteration requires computing the expected complete-data log-likelihood under the filtering distribution of $X$ given $Y$, which is analytically unsolvable and must be approximated. This necessitates a theoretical analysis of approximate EM sequences in infinite-dimensional spaces, driven by successive approximations of the filtering distributions.

Theoretical results on the convergence of EM algorithms are limited in scope even in the classical finite-dimensional setting, typically requiring restrictive regularity conditions \cite{Wu83, GuCh10}.  In our context, both the infinite-dimensional nature of the optimization problem and the necessity of approximating the filtering distribution at each iteration make the analysis substantially more intricate. However, we show that the structure of the underlying SDE, together with the properties of RKHS, leads to rigorous convergence results for both the exact and approximate EM sequences.  Our main theoretical contributions are \Cref{th:M-step-map-cont} and \Cref{th:AEM-prop}. The former establishes continuity of the M-step map with respect to approximations of the filtering distribution: specifically, if a sequence of approximating distributions converges weakly to the true filtering distribution of $X$ given $\obsY$, then the corresponding sequence of M-step optimizers converges in RKHS-norm to the true optimizer. The fact that strong (norm) convergence of the M-step optimizers holds despite requiring only weak convergence of the approximating filtering distributions is particularly noteworthy.  While this result ensures stepwise accuracy, it does not guarantee that approximation errors would not accumulate across iterations. Toward this end, \Cref{th:AEM-prop} shows that if the filtering distributions are approximated in the stronger sense of KL divergence, then the approximate EM sequence (defined in \eqref{eq:EM-approx}) retains key convergence properties of the original EM algorithm: the likelihood functional converges, and the iterates approach a stationary set $\stset$ containing the penalized maximum likelihood estimate (MLE) and any local max of the penalized likelihood functional.

For implementation, we employ a  Sequential Monte Carlo (SMC) algorithm ( i.e., Sequential Importance Sampling (SIS) with Resampling) \cite{Dou01, DoJo09} to approximate the filtering distribution and generate $L$ particle-paths at each EM iteration.  In contrast to MCMC methods, SMC offers significantly better scalability as the time horizon or the spacing between successive observation points increases. This advantage stems from the sequential generation of trajectory segments of the latent process $X$ and the recursive update of filtering weights. Together, they enable more efficient use of past computations in SMC framework leading to substantially lower computational cost per EM iteration.  In addition, the resampling step in SMC mitigates particle degeneracy, a common drawback of traditional importance sampling, by eliminating low-weight trajectories and concentrating computational effort on more informative particle-paths. The proposal distribution is crucial to the success of an SMC algorithm, and here we employ a carefully designed proposal based on a first-order linear SDE approximation, which performs particularly well in our highly nonlinear setting. The particle-paths generated by SMC are then used to construct a Monte Carlo approximation of the E-step value function.  Importantly, the subsequent M-step reduces to a penalized empirical risk minimization problem in the RKHS \cite{ScSm01}, which, thanks to a generalized representer theorem, admits a finite-sum kernel expansion (see \Cref{th:fin-sum-rep}). This ensures that each M-step yields a tractable closed form expression of the optimizer.  Crucially, because the kernel expansion must be recomputed at every EM iteration, and its number of terms grows with the data size, controlling its complexity is essential. To this end, we also develop a hybrid Bayesian variant of the EM algorithm that incorporates shrinkage priors to highlight the important coefficients in the kernel expansion while shrinking uninformative ones toward zero, thereby providing effective regularization, controlling model complexity, and improving numerical stability.  Our learning procedures are summarized in Algorithms \ref{algo:SMC}, \ref{algo:EM-SMC} and \ref{algo:Gibbs-EM-SMC}. 

This unified EM-SMC-RKHS framework enables principled and computationally tractable nonparametric estimation of SDE dynamics from sparse, noisy data. It brings together tools from stochastic analysis, stochastic filtering, Monte Carlo inference, and functional analysis, and provides a robust foundation for learning continuous-time dynamics in modern data-constrained settings. Several examples in Section \ref{sec:simu} demonstrate the accuracy of our method.
\vs{.2cm}

\np
{\em Layout:} The layout of the article is as follows. The mathematical description of the model is provided in \Cref{sec:frame}. The optimization problem for estimation is formulated in \Cref{sec:est}, followed by the EM method in RKHS and the associated theoretical results in \Cref{sec:EM}. The SMC approximation is discussed in \Cref{sec:smc-em} and \Cref{sec:SMC}. The learning algorithms are summarized in \Cref{sec:algo}, and numerical experiments are described in \Cref{sec:simu}. Some necessary auxiliary results are collected in \Cref{sec:app}.

\vs{.2cm}

\np
{\em Notations:} $\R^{m\times n}$ denotes the space of $m\times n$ real matrices. $\ve_{m\times n}: \R^{m\times n} \rt \R^{mn}$ will denote the vectorization function for $m\times n$ matrices. For a measurable space $(\mathbb{E}, \SC{E})$, $\msp^+(\mathbb E)$ and $\msp^+_{1}(\mathbb E)$ respectively  denote the sets of non-negative measures and probability measures on $\mathbb E$, equipped with the topology of weak convergence. Weak convergence of measures (and random variables) will be denoted by $\RT$. For two probability measures $\eta_{1}, \eta_{2} \in \msp^+_{1}(\mathbb E)$,  
\[
\RE(\eta_1 \| \eta_2) \dfeq 
\begin{cases}
\displaystyle \int_{\mathbb{E}} \log\left( \frac{d\eta_1}{d\eta_2}(v) \right) \, d\eta_1(v), & \text{if } \eta_1 \ll \eta_2, \\
\infty, & \text{otherwise}.
\end{cases}
\]
denotes the {\em relative entropy} or  {\em Kullback-Leibler (KL) divergence} of $\eta_1$ with respect to $\eta_2$.  $\|A\|_{\mrm{op}}$ denotes the operator or Spectral norm of a matrix $A$. For matrices $A \in \R^{m \times n}$ and $B \in \R^{m' \times n'}$, $A \ot B$ denotes the Kronecker product. When $A$ and $B$ are square (i.e., $m = n$ and $m' = n'$), the Kronecker sum is defined as $A \oplus B \equiv A \ot I_{n'} + I_n \ot B$. Weak and strong (norm) convergence in a Hilbert space will be denoted by $\stackrel{w}\Rt$ and $\stackrel{s}\Rt$, respectively.

\section{Mathematical Framework}\label{sec:frame}
  As mentioned in the introduction, we consider a discretized version of a $d$-dimensional SDE given by \eqref{eq:sde0}, where $\drift:\R^d \rt \R^d$ and $\diffus:\R^d \rt \R^{d\times d}$ and $W$ is a $d$-dimensional Brownian motion.  Thus our latent process will be a {\em discretized} SDE of the form
\begin{align}\label{eq:sde-disc-0}
	X(s_{n+1}) = X(s_n) +\drft(X(s_n)) \Delta+ \dffun(X(s_n)) \xi_{n+1} \sqrt \Delta,
\end{align}	
where the $\xi_n$ are i.i.d $\No_d(0,I)$-random variables. Here $\{s_n,\ n=0,1,2, \hdots\}$ with $s_0=0$ is a partition of $[0,\infty)$ with $\Delta=s_{n} - s_{n-1} \ll 1.$  $\BX_{[0,T]}$ will denote the trajectory of the chain $X$ in the time-interval $[0,T]$, that is, it is a random element in $\R^{d\times (n_0+1)}$ given by
\begin{align}
\BX_{[0,T]} \dfeq (X(s_0), X(s_1), X(s_2), \hdots, X(s_{N_0})).
\end{align}
Here $N_0$ is such that $s_{N_0} \leq T <s_{N_0+1}$. Notice that the density of the (discretized) trajectory  $\BX_{[0,T]}$  in $\R^{d\times (N_0+1)}$  is given by
\begin{align*}
	\Xden(\bx_{0:N_0}|\drft) = f_0(x_0)\prod_{n=1}^{N_0} \No_d\lf(x_n|x_{n-1}+\drft(x_{n-1})\Delta, a(x_{n-1}) \Delta\ri), \quad \bx_{0:N_0}=(x_0,\hdots, x_{N_0}) \in \R^{d\times (N_0+1)}. 
\end{align*}
Here $a(x)=\dffun(x)\dffun^\top(x)$.

%The unobserved high-frequency data $\bm{\procX}_{t_1:t_m} \dfeq (X(t_1), X(t_2),\hdots, X(t_m))$, where $\Delta=t_i-t_{i-1} \ll 1$ is driven by the above SDE. 

\np
{\em Observation Model}: We do not observe $X$ directly. Instead,  we are given sparse, partial and noisy data, $\by_{1:M_0} \equiv (y_1,y_2,\hdots, y_{M_0})$ --- a realization of the $\R^{d_0\times M}$-valued random observation matrix, $\bm{\obsY}_{t_1:t_{M_0}} = (\obsY(t_1), \obsY(t_2), \hdots, \obsY(t_{M_0}))$,
%\begin{align}\label{eq:obs-coll}
%\bm{\obsY}_{t_1:t_{M_0}} = (\obsY(t_1), \obsY(t_2), \hdots, \obsY(t_{M_0})),
%\end{align}
from the interval $[0,T]$, 
 where for each observation time $t_m$, $Y(t_m)$ takes values in $\R^{d_0}$ with $d_0 \leq d$, and the conditional density of $Y(t_m)$ given $X(t_m)$ is given by a known probability measure, $\obsden(\cdot| \procX(t_m))$. In other words,
\begin{align*}
\PP\lf(\obsY(t_m) \in A | X(t_m)\ri) = \int_A \obsden(u| \procX(t_m))\ du, \quad A \stackrel{m'ble} \subset \R^{d_0}.
\end{align*}
Our numerical experiments in Section \ref{sec:simu} will use the observation model of the form,
\begin{align}\label{eq:obs}
	\obsY(t_m) = \omat X(t_m) +\vep_m, \quad \text{ with } \vep_m \stackrel{iid}\sim \No_{d_0}(0, \Sigma_{noise}),
\end{align}
where $G \in \R^{d_0\times d}$ and the covariance matrix of the noise, $\Sigma_{noise} $, is p.d. In this case,  the conditional observation density $\obsden$ is given by $\obsden(\cdot| \procX(t_m)) = \No_{d_0} (\cdot|\omat X(t_m), \Sigma_{noise})$.

The data is called sparse since the time between successive observations, that is, $t_i -t_{i-1}$, is not necessarily small. We assume the functional form of the drift function of the SDE, $\drft$, is unknown, and our objective is to learn the {\em entire function} $\drft$ from these partial and noisy data-matrix, $\bm{\obsY}_{t_1:t_{M_0}}$.  For simplicity we will assume that the diffusion coefficient $\dffun$ is known. Following \cite{GaMiZh23}, the method can easily be extended to the case where $\s$ has a known functional form depending on an unknown parameter.

\section{Estimation procedure}\label{sec:est}
Without loss of generality, we will assume that $T=t_{M_0} =s_{N_0}$ and that the set of observation timepoints, $\{t_1,t_2,\hdots, t_{M_0}\} \subset \{s_0, s_1,\hdots, s_{N_0}\}$, i.e., for each $m = 1, \dots, M_0$ there exists $n_m \in \{0,1, \dots, N_0\}$ such that $t_m = s_{n_m}$. Thus the index set $\{n_1,n_2,n_3,\hdots, n_{M_0}\} \subset \{0,1,2,\hdots, N_0\}$ help to track the time points among $ \{s_0,s_1,\hdots, s_{N_0}\}$ where observations are available.

Given a realization, $\by_{1:M_0}$ of the random observation vector $\bm{\obsY}_{t_1:t_{M_0}}$, a natural estimation procedure is to minimize the following penalized negative log-likelihood functional of the drift function $\drft$ 
\begin{align}
	\label{eq:pen-neg-log-lik}
	\loss(\drft) \dfeq -\ell(\drift|\by_{1:M_0}) + \l \|\drift\|^2_{\Hsp_{\knl}}
\end{align}
over the function space, $\Hsp_{\knl}$, the RKHS of vector-valued functions corresponding to a (matrix-valued) kernel $\knl$ (see \Cref{def:matRK}), with the regularization parameter $\l>0$.  $\ell(\cdot|\by_{1:M_0})$ in \eqref{eq:pen-neg-log-lik} is the log-likelihood functional given $\BY_{t_1:t_{M_0}}=\by_{1:M_0}$ and is given by
\begin{align}\label{eq:log-like}
	\ell(\drft|\by_{1:M_0}) = \ln \Yden(\by_{1:M_0}|\drft) = \ln \int_{\R^{d\times (N_0+1)}} \XYden(\bx_{0:N_0}, \by_{1:M_0}| \drft) d\bx_{0:N_0},
\end{align}
where $\Yden(\cdot|\drft)$ is the marginal density of the observation-vector $\BY_{t_1:t_{M_0}}$, and $\XYden(\cdot, \cdot|\drft)$, the joint density of $(\BX_{[0,T]}, \BY_{t_1:t_{M_0}})$,  is given by
\begin{align*}
	\XYden(\bx_{0:N_0}, \by_{1:M_0}|\drft) =\Xden(\bx_{0:N_0}|\drft)\prod_{m=1}^{M_0}\obsden(y_m| \bx_{n_m}), \quad \bx_{0:N_0} \in  \R^{d\times (N_0+1)} , \by_{1:m} \in \R^{d_0\times M_0}. 
\end{align*}
 In other words, a penalized maximum likelihood estimator of the drift function $\drft$ is given by 
\begin{align}\label{eq:lik-min}
\hat \drft_{\mrm{pml}} \equiv \hat \drft_{\mrm{pml}}(\by_{1:M_0}) \dfeq \argmin_{\drift \in \Hsp_{\knl}} \loss(\drft).
\end{align}

% The EM algorithm relies on $ \filtdist_{b}(\cdot| \BY_{t_1:t_{M_0}})$, the filtering (conditional) distribution of $\BX_{[0,T]}$ given $\BY_{t_1:t_{M_0}} = \by_{1:M_0}$, whose density, which by a slight abuse of notation we also denote by $ \filtden_{b}$,  is given by

\subsection{EM algorithm on RKHS} \label{sec:EM}
Clearly, $\Yden(\cdot|\drft)$, or equivalently, the log-likelihood functional $\ell(b|\by_{1:M_0}) $ is not available in closed form making the minimization problem in \eqref{eq:lik-min} particularly difficult. 
Our approach to solve the optimization problem in \eqref{eq:lik-min} is via an EM algorithm in the RKHS, $\Hsp_{\knl}$.

Toward this end, let $ \filtdist_{b}(\cdot| \by_{1:M_0})$ denote the filtering (conditional) distribution of $\BX_{[0,T]}$ given $\BY_{t_1:t_{M_0}} = \by_{1:M_0}$. This distribution admits a density, which, by a slight abuse of notation, we also denote by $ \filtden_{b}$, and it is given by
\begin{align}\label{eq:filtden-exp}
	\begin{aligned}
		\filtden_{\drft}(\bx_{0:N_0}| \by_{1:M_0}) =&\ \XYden(\bx_{0:N_0}, \by_{1:M_0}|\drft, \dffun) \big/ \Yden(\by_{1:M_0}|\drft)\\
		=&\  \Xden(\bx_{0:N_0}|\drft)\prod_{m=1}^{M_0}\obsden(y_m| \bx_{n_m}) \big/ \Yden(\by_{1:M_0}|\drft). 
	\end{aligned}
\end{align}	
Notice that the identity
$
	\XYden(\bx_{0:N_0}, \by_{1:M_0}|\drft) =  \Yden(\by_{1:M_0}|\drft) \filtden_{\drft}(\bx_{0:N_0}| \by_{1:M_0})
$
shows that for any probability measure $\gmeas$ on $\R^{d\times (N_{0}+1)}$,
\begin{align}\label{eq:lik-exp-meas}
	\ell(b|\by_{1:M_0}) =&\ \int \ln \XYden(\bx_{0:N_0}, \by_{1:M_0}|\drft) \gmeas(d\bx_{0:N_0})+ \RE\lf(\gmeas\|\filtden_{b}(\cdot| \by_{1:M_0})\ri)+\ent(\gmeas),
\end{align}
where recall that $\RE(\gmeas_1\|\gmeas_2)$ denotes the relative entropy of the $\gmeas_1 \in \msp^+_{1}(\R^{d\times (N_{0}+1)})$ with respect to the $\gmeas_2 \in \msp^+_{1}(\R^{d\times (N_{0}+1)}),$ and $\ent(\gmeas)$ is the entropy of the probability measure $\gmeas \in \msp^+_{1}(\R^{d\times (N_{0}+1)}).$ Thus for any probability measure $\gmeas$, the loss function $\loss: \Hsp_{\knl} \rt [0,\infty)$ in \eqref{eq:pen-neg-log-lik} satisfies
\begin{align}
	\label{eq:loss-meas}
	\begin{aligned}
		\loss(\drft) =&\ \tilde \risk(\drft,\gmeas)-\RE\lf(\gmeas\|\filtden_{b}(\cdot| \by_{1:M_0})\ri)-\ent(\gmeas)\\
		=&\ \risk(\drft,\gmeas)-\RE\lf(\gmeas\|\filtden_{b}(\cdot| \by_{1:M_0})\ri)-\ent(\gmeas) -   \const(\gmeas, \by_{1:M_0}),
	\end{aligned}
\end{align}
where the risk functions $\tilde \risk,\risk: \Hsp_{\knl}\times \msp^+_{1}(\R^{d\times (N_{0}+1)}) \rt [0,\infty]$ are defined by
\begin{align}
\label{eq:risk-funct}
\begin{aligned}
\tilde \risk(\drft,\gmeas)\dfeq&\ - \int_{\R^{d\times (N_{0}+1)}}  \ln \XYden(\bx_{0:N_0},  \BY_{t_1:t_{M_0}}|\drft) \gmeas(d\bx_{0:N_0}) + \l \|\drft\|^2_{\Hsp_{\knl}}\\
\risk(\drft,\gmeas)\dfeq&\ - \int_{\R^{d\times (N_{0}+1)}}  \ln \Xden(\bx_{0:N_0}|\drft) \gmeas(d\bx_{0:N_0})   + \l \|\drft\|^2_{\Hsp_{\knl}},
\end{aligned}
\end{align}
and the (non-negative) constant $\const(\gmeas, \by_{1:M_0}) \dfeq \sum_{m=1}^{M_0}\int 
\ln\obsden(y_m| \bx_{n_m})\gmeas(d\bx_{0:N_0})$, which does not depend on $\drft,$ comes from the identity
\begin{align*}
\tilde \risk(\drft,\gmeas) = \risk(\drft,\gmeas) - \const(\gmeas, \by_{1:M_0}).
\end{align*}
\eqref{eq:loss-meas} shows that  for any fixed probability measure $\gmeas$, $\risk(\cdot,\gmeas)$ serves as an upper bound of the penalized negative log-likelihood functional $\loss(\cdot)$, and instead of minimizing of $\loss(\cdot)$ ---  an intractable problem --- we  consider the surrogate optimization problem of minimizing the risk function $\risk(\cdot, \gmeas)$  --- which, as we will show later, is feasible (up to a good approximation) using RKHS theory.

  For any fixed probability measure $\gmeas$,  denote the minimizer of $\risk(\cdot,\gmeas)$ by
\begin{align}
\label{eq:opt-def-1}
\Phi(\gmeas) \dfeq\argmin_{\drft \in \Hsp_{\knl}} \tilde \risk(\drft,\gmeas) = \argmin_{\drft \in \Hsp_{\knl}} \risk(\drft,\gmeas).
\end{align}
Note that while $\Phi(\gmeas)$ serves as a substitute for  $\hat \drft_{\mrm{pml}}$ for each $\gmeas$, the goal is to find a suitable $\gmeas$ that ensures $\Phi(\gmeas)$ is computable and close to $\hat \drft_{\mrm{pml}}.$ As is intuitively clear, the optimal choice is the probability measure, $\gmeas^{\mrm{opt}}  \equiv \argmin_{\gmeas}\|\hat \drft_{\mrm{pml}} -\Phi(\gmeas)\|_{\knl} = \filtden_{\hat \drft_{\mrm{pml}}}(\cdot|\by_{1:M_0})$, which follows because of the following identity (proven in  \Cref{cor:prop-st-set}) $\hat \drft_{\mrm{pml}} =\Phi(\filtden_{\hat \drft_{\mrm{pml}}}(\cdot|\by_{1:M_0}))$. In other words, $\hat \drft_{\mrm{pml}}$ is a member of  $\stset$, the {\em stationary or invariant set} of $\Phi$, defined by
\begin{align}
	\label{eq:stat-set}
	\stset \dfeq \big\{\drft \in \Hsp_{\knl}: \drft = \Phi(\filtden_{\drft}(\cdot| \by_{1:M_0}))\big\}.
\end{align}
The optimal probability measure is of course not computable as it depends on the intractable $\hat \drft_{\mrm{pml}}$ in the first place, but it shows the importance of $\stset.$

 The EM algorithm provides a systematic approach to obtain a potential close-to-optimal choice of  $\gmeas$ by iteratively solving the minimization problem \eqref{eq:opt-def-1} starting with $\gmeas_{0}=\filtden_{\drft_0}(\cdot| \by_{1:M_0})$ for some initial guess of the drift function, $\drft_0$. We will show that in our infinite-dimensional framework an EM-sequence, defined below, converges to a member of the stationary set $\stset$ --- a natural target for the EM algorithm, indicating that further iteration will not yield new estimates. If  $\stset$ is singleton, a condition which unfortunately is hard to verify in practice, any EM sequence is guaranteed to converge to $\hat \drft_{\mrm{pml}}.$
 
  %First, by a slight abuse of notation, denote
%\begin{align}
%\label{eq:opt-def-1}
%\Phi_{\by_{1:M_0}}(\drft_0) \equiv \Phi(\filtden_{\drft_0}(\cdot| \by_{1:M_0})) = \argmin_{\drft \in \Hsp_{\knl}} \risk(\drft,\filtden_{\drft_0}(\cdot| \by_{1:M_0})), \quad \drft_0 \in \Hsp_{\knl}
%\end{align}
\begin{definition}
An EM-sequence $\{\drft_k\equiv \drft_k(\drft_0, \by_{1:M_0}):k\geq 0\} \subset  \Hsp_{\knl}$ for an initial guess $\drft_0 \in \Hsp_{\knl}$ is defined by the recursion
\begin{align} \label{eq:EM-0}
	 \drft_k\dfeq&\ \Phi(\filtden_{\drft_{k-1}}(\cdot| \by_{1:M_0}))  = \argmin_{\drft \in \Hsp_{\knl}} \risk\big(\drft,\filtden_{\drft_{k-1}}(\cdot| \by_{1:M_0})\big), \quad k\geq 1.
\end{align}
\end{definition}	
\np
The existence of the above infinite-dimensional EM sequence of course depends on the map  $\Phi: \msp^+_{1}(\R^{d\times (N_{0}+1)}) \rt \Hsp_{\knl}$ being well-defined in a suitable subspace of $\msp^+_{1}(\R^{d\times (N_{0}+1)})$, which we do in \Cref{lem:ext-unq-min}  (also see \Cref{cor:exst-em-map}).

\np
{\bf Approximate EM sequence:} Because of the intractability of the marginal density $\Yden(\by_{1:M_0}|\drft)$, the filtering distribution $\filtden_{\drft}(\cdot| \by_{1:M_0})$, for a given drift function $\drft$, is typically not available in closed form, rendering direct implementation of the EM algorithm impossible. To resolve this, for each iteration of the EM algorithm, we need to consider tractable approximation of the filtering distribution. 

For each $\drft \in \Hsp_{\knl}$, let $\big\{\hat \filtden^{(L)}_{\drft}(\cdot| \by_{1:M_0})\big\}$ be an approximating scheme for $\filtden_{\drft}(\cdot| \by_{1:M_0})$; that is, $\hat \filtden^{(L)}_{\drft}(\cdot| \by_{1:M_0})$ converges to $\filtden_{\drft}(\cdot| \by_{1:M_0})$ in a suitable sense as $L \rt \infty.$
Starting from an initial guess $\drft_0 \in \Hsp_{\knl}$, construct an {\em approximate EM sequence} $\big\{\hat\drft_k\equiv \hat\drft_k(\drft_0, \by_{1:M_0}):k\geq 0\big\}$ by the following recursion
\begin{align} \label{eq:EM-approx}
	\hat \drft_0\equiv \drft_0, \quad \hat \drft_k\dfeq&\ \Phi\Big(\hat \filtden^{(L_{k-1})}_{\hat \drft_{k-1}}(\cdot| \by_{1:M_0})\Big)  = \argmin_{\drft \in \Hsp_{\knl}} \risk\Big(\drft,\hat \filtden^{(L_{k-1})}_{\hat \drft_{k-1}}(\cdot| \by_{1:M_0})\Big), \quad k\geq 1.
\end{align}
where $L_{k-1}$ is appropriately chosen at each iteration $k$.

If $\hat \filtden^{(L)}_{\drft}(\cdot| \by_{1:M_0}) \stackrel{L \rt \infty}\RT \filtden_{\drft}(\cdot| \by_{1:M_0})$, and the current iterate is $\drft$, then \Cref{th:M-step-map-cont} guarantees that   replacing $\filtden_{\drft}(\cdot| \by_{1:M_0})$ with $\hat \filtden^{(L)}_{\drft}(\cdot| \by_{1:M_0})$ in the M-step is justified as $\Phi\Big(\hat \filtden^{(L)}_{\drft}(\cdot| \by_{1:M_0})\Big)$ converges to $\Phi\Big(\filtden_{\drft}(\cdot| \by_{1:M_0})\Big)$ (actually in strong sense). Although the error at each iteration can be kept small by choosing a large $L$, this does not immediately guarantee that the resulting {\em approximate EM sequence}, $\{\hat \drft_k\}$, will exhibit the desirable properties expected of the actual EM sequence $\{ \drft_k\}$ --- such as convergence to a stationary set and convergence of the loss function along the sequence. A delicate analysis is needed to track the propagation of error along the sequence, and one of the main theoretical results of the paper, \Cref{th:AEM-prop}, shows that these properties do hold for $\{\hat \drft_k\}$, provided we have an approximating scheme  $\big\{\hat \filtden^{(L)}_{\drft}(\cdot| \by_{1:M_0})\big\}$ that approximates $\filtden_{\drft}(\cdot| \by_{1:M_0})$ for every $\drft \in \Hsp_{\knl}$ in the stronger sense of KL-divergence (relative entropy).

\subsection*{Theoretical Analysis}
Our primary results are \Cref{th:M-step-map-cont} and \Cref{th:AEM-prop}.
We begin by stating the assumptions that are necessary for the subsequent results. Not all of these assumptions are needed for every result; we will indicate the specific conditions relevant to each case. We will always assume that the functions $\drift$ and $\diffus$ are such that the \eqref{eq:sde0}  admits a unique strong solution. 
\begin{assumption}\label{assum:sde-knl}
%There exist constants $\const_{\dffun}, \const_{a}, \const_{\knl}, \obsdenbd \geq 0$ and exponents $\sexp,\aexp, \kexp \geq 0$ such that
The following hold for some constants $\const_{\dffun}, \const_{a}, \const_{\knl}, \obsdenbd \geq 0$ and exponents $\sexp,\aexp, \kexp \geq 0$.
\begin{enumerate}[label=(\alph*), ref=(\alph*), resume=mylist]
 \item\label{item:ident} For $\drft, \tilde \drft \in \Hsp_{\knl}$, $\Xden(\bx_{0:N_0}|\drft) = \Xden(\bx_{0:N_0}|\tilde \drft)$ for all  $\bx_{0:N_0} \in \R^{d \times (N_0+1)}$ implies that $\drft = \tilde \drft$.
\item \label{item-assum-knl} $\knl$ is  continuous in each argument, the mapping $u \in \R^d \Rt \knl(u,u) \in \R^{d\times d}$ is continuous, and $\|\knl(u,u)\|_{\mrm{op}} \leq \const_{\knl}(1+\|u\|^{\kexp})  \ u \in \R^d;$
\end{enumerate}	
	\begin{multicols}{2}
		\begin{enumerate}[label=(\alph*), ref=(\alph*), resume=mylist]
			\item \label{item-assum-dffun}  $\|\dffun(u)\|_{\mrm{op}}  \leq \const_{\dffun}(1+\|u\|^{\sexp}),  \ u \in \R^d,$
			\item \label{item-assum-a}  $\|a^{-1}(u)\|_{\mrm{op}}  \leq \const_{a}(1+\|u\|^{\aexp}),  \ u \in \R^d,$
		\end{enumerate}
	\end{multicols}
\begin{enumerate}[label=(\alph*), resume=mylist]
\item \label{item-assum-obsden}  $\obsden(y|x) \leq \obsdenbd, \ y \in \R^{d_0}, \ x \in \R^d$,
 % \item \label{item:ident} if for $\drft,\tilde \drft \in \Hsp_{\knl}$, $\Xden(\bx_{0:N_0}|\drft) = \Xden(\bx_{0:N_0}|\drft)$ for all $\bx_{0:N_0} \in \R^{d\times (N_{0}+1)}$, then $\drft =\tilde \drft.$
\end{enumerate}
\end{assumption}

\Cref{assum:sde-knl}-\ref{item:ident} is an identifiability condition ensuring that distinct drift functions induce different distributions for $X$, which, as shown below, is equivalent to the filtering distributions of $X$ given the data $\by_{1:M_0}$ also being distinct.
\begin{lemma}\label{lem:eqv-ident}
\Cref{assum:sde-knl}-\ref{item:ident} is equivalent to the condition that $\filtden_{\drft}(\bx_{0:N_0}| \by_{1:M_0}) = \filtden_{\tilde\drft}(\bx_{0:N_0}| \by_{1:M_0})$  for all  $\bx_{0:N_0} \in \R^{d \times (N_0+1)}$ implies  $\drft = \tilde \drft$.
\end{lemma}

\begin{proof}
Assume that $\filtden_{\drft}(\bx_{0:N_0}| \by_{1:M_0}) = \filtden_{\tilde\drft}(\bx_{0:N_0}| \by_{1:M_0})$. Since $\obsden$ does not depend on the drift function, we have from \eqref{eq:filtden-exp} that 
\begin{align*}
\Xden(\bx_{0:N_0}|\drft) \big/ \Yden(\by_{1:M_0}|\drft) = \Xden(\bx_{0:N_0}|\tilde\drft) \big/ \Yden(\by_{1:M_0}|\tilde\drft).
\end{align*}
Integrating both sides with respect to $\bx_{0:N_0}$, we get $\Yden(\by_{1:M_0}|\drft) = \Yden(\by_{1:M_0}|\tilde\drft)$ which in turn shows that $\Xden(\bx_{0:N_0}|\drft) = \Xden(\bx_{0:N_0}|\tilde\drft)$. The reverse direction follows similarly.
\end{proof}

\np
We now introduce the subspaces of probability measures that are relevant for our results.
\begin{definition}\label{def:meas-mmt-class}
For  exponents $p_0\geq 0,$ and constant $\cnst>0$, let $\msp^{+,p_0}_{1, \cnst}(\R^{d\times (N_{0}+1)}), \msp^{+,p_0}_{1}(\R^{d\times (N_{0}+1)})\subset \msp^+_{1}(\R^{d\times (N_{0}+1)})$  be the subset of probability measures defined by
%\begin{align*}
%\msp^{+,p_0,p_1}_{1, \cnst}(\R^{d\times (N_{0}+1)}) \dfeq\lf\{\gmeas \in \msp^+_{1}(\R^{d\times (N_{0}+1)}): \int_{\R^d} |x_k|^{p_0} \gmeas_k(dx_k) \leq \cnst, \ \int_{\R^d} \|\knl(x_k,x_k)\|^{p_1} \gmeas_k(dx_k) \leq \cnst\ri\}.
%\end{align*}
\begin{align*}
	\msp^{+,p}_{1, \cnst}(\R^{d\times (N_{0}+1)}) \dfeq&\ \lf\{\gmeas \in \msp^+_{1}(\R^{d\times (N_{0}+1)}): \int_{\R^d} |x_k|^{p} \gmeas_k(dx_k) \leq \cnst,\ 0\leq k\leq N_0 \ri\},\\
	\msp^{+,p}_{1}(\R^{d\times (N_{0}+1)}) \dfeq&\ \lf\{\gmeas \in \msp^+_{1}(\R^{d\times (N_{0}+1)}): \int_{\R^d} |x_k|^{p} \gmeas_k(dx_k) < \infty,\ 0\leq k\leq N_0 \ri\}\\
	=&\ \cup_{\cnst\geq 0} \msp^{+,p}_{1, \cnst}(\R^{d\times (N_{0}+1)}).
\end{align*}
Here for $\gmeas \in  \msp^+_{1}(\R^{d\times (N_{0}+1)})$, $\gmeas_k \in \msp^+_{1}(\R^{d})$ denotes the $k$-th marginal distribution of $\gmeas.$
\end{definition}

Notice that $\msp^{+,p}_{1, \cnst}(\R^{d\times (N_{0}+1)}) $ is a closed set. Indeed, if $\{\gmeas^{(L)} \} \subset \msp^{+,p}_{1, \cnst}(\R^{d\times (N_{0}+1)})$ and $\gmeas^{(L)} \stackrel{L \rt \infty} \RT \gmeas$ then by Fatou's lemma and lower semicontinuity of the mapping $x \Rt \|x\|^p$,
$\int_{\R^d} |x_k|^{p} \gmeas_k(dx_k)  \leq \liminf_{L\rt \infty} \int_{\R^d} |x_k|^{p} \gmeas^{(L)}_k(dx_k) \leq \cnst.$

We now show for each fixed $\gmeas \in\msp^+_{1}(\R^{d\times (N_{0}+1)}) $ satisfying suitable moment condition, $\Phi(\gmeas)$ is well-defined in the sense the mapping $\drft \in  \Hsp_{\knl} \rt \risk(\drft,\gmeas) \in [0,\infty)$ attains a unique  minimizer. 

\begin{lemma}\label{lem:ext-unq-min}
Suppose that \Cref{assum:sde-knl}-\ref{item-assum-a} holds.
For $\l>0$, let $\risk:\Hsp_{\knl}\times \msp^+_{1}(\R^{d\times (N_{0}+1)}) \rt [0,\infty)$ be defined by \eqref{eq:risk-funct}. Then for any fixed probability measure $\gmeas \in \msp_{1}^{+,2+\aexp}(\R^{d\times (N_{0}+1)}),$ there exists a unique (global) minimizer $\Phi(\gmeas) \in \Hsp_{\knl}$ of the function $\risk(\cdot, \gmeas):\Hsp_{\knl} \rt [0,\infty).$
\end{lemma}

\begin{proof}
Let $\risk_{*}(\gmeas) = \inf_{\drft \in \Hsp_{\knl} } \risk(\drft,\gmeas)$. Clearly $\risk_{*}(\gmeas) < \infty,$ as  $\risk(\drft\equiv 0, \gmeas)<\infty$ for $\gmeas \in \msp_{1}^{+,2+\aexp}(\R^{d\times (N_{0}+1)}).$ Then there exists a sequence $\{\drft^{(n)}\}$ such that $\risk(\drft^{(n)},\gmeas) \rt \risk_{*}(\gmeas) $, as $ n \rt \infty.$
	Notice that this implies the sequence $\{\|\drft^{(n)}\|_{\Hsp_{\knl}}\}$ is bounded. Indeed, if this is not true then $\limsup_{n \rt \infty} \|\drft^{(n)}\|_{\Hsp_{\knl}} = \infty$. Since $-\ln \Xden(\bx_{0:N_0}|\drft) \geq 0$, this then implies that
	$$\risk_{*}(\gmeas) = \lim_{n\rt \infty} \risk(\drft^{(n)},\gmeas) \geq \l\limsup_{n\rt \infty} \|\drft^{(n)}\|^2_{\Hsp_{\knl}}   = \infty,$$
 which contradicts the fact that $\risk_{*}(\gmeas)<\infty$. Consequently, by Banach-Alaoglu (and Eberlein-Smulian theorem) there exists an $\drft_* \in \Hsp_{\knl}$ and a subsequence $\{n_k\}$ such that $\drft^{(n_k)} \stackrel{w} \Rt \drft_*$. By the weak sequential l.s.c. of $\risk(\cdot,\eta)$ we conclude
	\begin{align*}
		\risk_{*}(\gmeas) = \lim_{k\rt \infty} \risk(\drft^{(n_k)},\gmeas) \geq \risk(\drft_*, \gmeas) \geq \inf_{\drft \in \Hsp_{\knl}} \risk(\drft,\gmeas) = \risk_{*}(\gmeas).
	\end{align*}
	This proves that the infimum of $\risk(\cdot, \gmeas)$ is attained at $\drft_*$. The uniqueness simply follows from strict convexity of the mapping $\drft \rt \risk(\drft, \gmeas)$.
\end{proof}

The following proposition plays a key role in the proofs of our main results.

\begin{proposition} \label{prop:conv-result}
  Consider the sequences $\{\gmeas, \gmeas^{(L)}:\ L\geq 1\} \subset \msp^{+}_{1}(\R^{d\times (N_{0}+1)}),$  $\{\drft,\drft^{(L)}: L\geq 1\} \subset \Hsp_{\knl}$ such that as $L\rt \infty$, $\gmeas^{(L)} \RT \gmeas, \quad \drft^{(L)}\stackrel{w} \Rt \drft$. 
Then the following hold.
\begin{enumerate}[label=(\roman*), ref=(\roman*)]
\item  \label{item:conv-filtden-like} $\Yden(\by_{1:M_0}|\drft^{(L)})  \stackrel{L\rt \infty}\Rt \Yden(\by_{1:M_0}|\drft) $ \big(equivalently, $\ell(\drft^{(L)}|\by_{1:M_0}) \stackrel{L\rt \infty}\Rt \ell(\drft|\by_{1:M_0})$\big), and \\
 $\filtden_{\drft^{(L)}}(\cdot| \by_{1:M_0}) \stackrel{L\rt \infty}\Rt \filtden_{\drft}(\cdot| \by_{1:M_0})$ pointwise and in $L^{1}(\R^{d\times (N_{0}+1)})$.

\item \label{item:conv-ln-f} Suppose that \Cref{assum:sde-knl}-\ref{item-assum-knl} \& \ref{item-assum-a} hold and for some $p>2 \kexp\vee1+\aexp$ and $\cnst>0$, $\{\gmeas^{(L)}:\ L\geq 1\} \subset \msp^{+,p}_{1, \cnst}(\R^{d\times (N_{0}+1)}).$ Then 
$$ \int_{\R^{d\times (N_{0}+1)}}  \ln \Xden(\bx_{0:N_0}|\drft^{(L)}) \gmeas^{(L)}(d\bx_{0:N_0}) \stackrel{L\rt \infty}\Rt \int_{\R^{d\times (N_{0}+1)}}  \ln \Xden(\bx_{0:N_0}|\drft)\gmeas(d\bx_{0:N_0}).$$

%\item $\dst \RE\lf(\filtden_{\drft^{(L)}}(\cdot| \by_{1:M_0})\| \gmeas^{(L)}\ri) \stackrel{L\rt \infty}\Rt \RE\lf(\filtden_{\drft}(\cdot| \by_{1:M_0})\| \gmeas\ri).$ 
\end{enumerate}

\end{proposition}

\begin{proof}
\np
\ref{item:conv-filtden-like} 
First, observe that $ \drft^{(L)}\stackrel{w} \Rt \drft$ as $L \rt \infty$ implies pointwise convergence of $\drft^{(L)}$ to $\drft$; indeed, letting $e_j, \ j=1,2,\hdots,d$ denote the canonical unit vectors in $\R^d$, we see that for each $j=1,2, \hdots, d,$
 $$\drft^{(L)}_j(u) = \drft^{(L)}(u)^\top e_j =\<\knl(u,\cdot)e_j, \drft^{(L)}\> \stackrel{L \rt \infty}\Rt \<\knl(u,\cdot)e_j, \drft\> = \drft(u)^\top e_j= \drft_j(u), \quad u \in \R^d. $$
This shows that $ \Xden(\cdot|\drft^{(L)}) \stackrel{L\rt \infty} \Rt  \Xden(\cdot|\drft) $ pointwise, and hence, also in $L^{1}(\R^{d\times (N_{0}+1)})$ by Scheffe's lemma.
% $\int_{ \R^{d\times (N_{0}+1)} }\lf | \Xden(\bx_{0:N_0}|\drft^{(L)}) -  \Xden(\bx_{0:N_0}|\drft)\ri| d \bx_{0:N_0} \stackrel{L\rt \infty} \Rt 0$. 
Therefore,
\begin{align*}
	\lf|\Yden(\by_{1:M_0}|\drft^{(L)})-\Yden(\by_{1:M_0}|\drft)\ri| \leq&\ \int_{ \R^{d\times (N_{0}+1)} }\lf | \Xden(\bx_{0:N_0}|\drft^{(L)}) -  \Xden(\bx_{0:N_0}|\drft)\ri| \prod_{m=1}^{M_0}\obsden(y_m| \bx_{n_m}) d \bx_{0:N_0}\\
	\leq &\  \obsdenbd^{M_0} \int_{ \R^{d\times (N_{0}+1)} }\lf | \Xden(\bx_{0:N_0}|\drft^{(L)}) -  \Xden(\bx_{0:N_0}|\drft)\ri| d \bx_{0:N_0} \stackrel{L\rt \infty} \Rt 0.
\end{align*}	
The pointwise convergence of the sequence of filtering densities, $\filtden_{\drft^{(L)}}(\cdot| \by_{1:M_0})$, now follows immediately from their expressions in \eqref{eq:filtden-exp}, and the convergence in $L^{1}(\R^{d\times (N_{0}+1)})$ again by Scheffe's lemma.

\np
\ref{item:conv-ln-f} Recall that $a=\s\s^\top$ and notice
 \begin{align}\label{eq:ln-f-exp}
\begin{aligned}
	\ln \Xden (\bx_{0:N_0}|\drft^{(L)}) =&\ \ln f_0(x_0) - N_0(d+1)\ln(2\pi)/2+ \f{1}{2}\sum_{n=1}^{N_0}\ln\lf(\det\lf(\Delta^{-1} a^{-1}(x_{n-1})\ri)\ri)\\
&\   -\Delta^{-1}\sum_{n=1}^{N_0}(x_n-x_{n-1}-\drft^{(L)}(x_{n-1})\Delta)^\top a^{-1}(x_{n-1})(x_n-x_{n-1}-\drft^{(L)}(x_{n-1})\Delta)^\top
\end{aligned}
\end{align}
We first show that the mapping $\bx_{0:N_0} \in \R^{d\times (N_{0}+1)} \rt \ln \Xden(\bx_{0:N_0}|\drft^{(L)})$ satisfies the assumption \ref{item:uni-cmpt} of Lemma \ref{lem:uni-int} that is, $\ln \Xden(\cdot|\drft^{(L)}) \stackrel{L\rt \infty} \Rt \ln \Xden(\cdot|\drft) $ uniformly over compact sets of $\R^{d\times (N_{0}+1)}$. It's clear from \eqref{eq:ln-f-exp} that this holds if we can establish that $\drft^{(L)} \stackrel{L\rt \infty} \Rt \drft$ uniformly over compact sets of $\R^d$. We have already observed $ \drft^{(L)}\stackrel{w} \Rt \drft$  implies $ \drft^{(L)} \stackrel{L\rt \infty}\Rt \drft$ pointwise.  Also, by the uniform boundedness principle, $\sup_{L}\|\drft^{(L)}\|_{\knl} < \infty$. This, together with the continuity assumptions on $\knl$ and \eqref{eq:rkhs-l2norm-est}, now implies that
\begin{align*}
 \sup_{L}\|\drft^{(L)}(u) - \drft^{(L)}(u')\|   \leq&\ \|\knl(u,u)-2\knl(u,u')+\knl(u',u')\|_{\mrm{op}} \sup_{L}\|\drft^{(L)}\|_{\Hsp_{\knl}} \stackrel{u'\rt u} \Rt 0;
\end{align*}
that is, $\{\drft^{(L)}\}$ is equicontinuous. It now follows by the Arzella-Ascolli theorem that  $\drft^{(L)} \stackrel{L\rt \infty} \Rt \drft$ uniformly over compact sets of $\R^d$.

 We next show that the mapping $\bx_{0:N_0} \in \R^{d\times (N_{0}+1)} \rt \ln \Xden(\bx_{0:N_0}|\drft^{(L)})$ satisfies the assumption  \ref{item:uni-int} of Lemma \ref{lem:uni-int}. Toward this end observe that
\begin{align*}
|\ln \Xden(\bx_{0:N_0}|\drft^{(L)})| \leq &\ \const_0\sum_{n=0}^{N_0}\Big(1+\|x_n\|^{2+\aexp}+\|\drft^{(L)}\|_{\Hsp_{\knl}} \|x_n\|^{1+\aexp} \|\knl(x_n,x_n)\|_{\mrm{op}}\\
&  \qquad  +\|\drft^{(L)}\|^2_{\Hsp_{\knl}} \|x_n\|^{\aexp} \|\knl(x_n,x_n)\|^2_{\mrm{op}}+ \ln(1+\|x_n\|)\Big)\\
\leq&\ \const_1\sum_{n=0}^{N_0}\lf(1+\|x_n\|^{2\kexp\vee 1+\aexp}\ri),
\end{align*}
where for the last inequality we used \Cref{assum:sde-knl}-\ref{item-assum-knl} \& \ref{item-assum-a} and  the previously stated fact, $\sup_{L}\|\drft^{(L)}\|_{\knl} < \infty$.
Thus $\sup_L|\ln \Xden(\bx_{0:N_0}|\drft^{(L)})|\big/\Lambda(\bx_{0:N_0}) \Rt 0$ as $|u| \rt \infty$, where $\Lambda$, defined by, $\Lambda(\bx_{0:N_0}) \dfeq \sum_{n=0}^{N_0}(1+\|x_n\|^{p})$ (with $p$ as in \ref{item:conv-ln-f})
satisfies \eqref{assum-1} of Lemma \ref{lem:uni-int} because of the hypothesis on $\{\gmeas^{(L)}\}$. The assertion in \ref{item:conv-ln-f} now follows from Lemma \ref{lem:uni-int}.
\end{proof}

\begin{corollary}\label{cor:conv-RE}
Let $\drft^{(L)}\stackrel{w} \Rt \drft$ in $\Hsp_{\knl}$ as $L \rt \infty.$ Suppose that \Cref{assum:sde-knl}-\ref{item-assum-knl}, \ref{item-assum-dffun} \& \ref{item-assum-obsden}  hold. Then for for some constant $\tilde \cnst \geq 0$,  $\{\filtden_{\drft^{(L)}}(\cdot| \by_{1:M_0})\} \subset \msp^{+,p_0}_{1, \tilde \cnst }(\R^{d\times (N_{0}+1)}).$ 

Let  $\{\drft^{(L)}_1\} \subset  \Hsp_{\knl}$ be another family such that $\drft^{(L)}_1\stackrel{w} \Rt \drft_1$ as $L \rt \infty$. Additionally, suppose that \Cref{assum:sde-knl}-\ref{item-assum-a}  holds. Then $\dst \RE\lf(\filtden_{\drft^{(L)}}(\cdot| \by_{1:M_0})\big\| \filtden_{\drft^{(L)}_1}(\cdot| \by_{1:M_0})\ri) \stackrel{L\rt \infty}\Rt \RE\lf(\filtden_{\drft}(\cdot| \by_{1:M_0})\drft\| \filtden_{\drft_1}(\cdot| \by_{1:M_0})\ri).$ 
 \end{corollary}

\begin{proof}
Fix any $p>0$. As noted before, because of the uniform boundedness principle, $\drft^{(L)}\stackrel{w} \Rt \drft$ implies that   $\sup_{L}\|\drft^{(L)}\|_{\knl} < \infty$. It  follows by \Cref{lem:mmt-bd-X} that  for any $n\leq N_0$ and $L\geq 1$
\begin{align*}
	\int_{\R^{d\times (N_{0}+1)}} |x_n|^p \XYden(\bx_{0:N_0}, \by_{1:M_0}|\drft, \dffun)\ d\bx_{0:N_0} \leq &\  \obsdenbd^{M_0} \int_{ \R^{d\times (N_{0}+1)} }|x_n|^p \Xden(\bx_{0:N_0}|\drft, \dffun)\ d\bx_{0:N_0}\\
	 =&\ \obsdenbd^{M_0}\ \EE_{\drft^{(L)}}|X(s_{n})|^p \leq\ \obsdenbd^{M_0}\ \const_{p,1}(N_0).
\end{align*}
Since   $\Yden(\by_{1:M_0}|\drft^{(L)})  \stackrel{L\rt \infty}\Rt \Yden(\by_{1:M_0}|\drft) \neq 0$ by \Cref{prop:conv-result}-\ref{item:conv-filtden-like} both the sequences, $\{\Yden(\by_{1:M_0}|\drft^{(L)})\}$ and $\{1/\Yden(\by_{1:M_0}|\drft^{(L)})\}$ are bounded.  It follows from the expression of $\filtden_{\drft^{(L)}}(\cdot| \by_{1:M_0})$ (see \eqref{eq:filtden-exp}) that for some constant $\tilde \cnst \equiv \tilde \cnst_{M_0,p}$
\begin{align*}
\sup_{L}\int_{\R^{d\times (N_{0}+1)}} |x_n|^p\ \filtden_{\drft^{(L)}}(\bx_{0:N_0}| \by_{1:M_0})\ d\bx_{0:N_0} < \tilde \cnst .
\end{align*}
which proves that the family of probability measures, $\{\filtden_{\drft^{(L)}}(\cdot| \by_{1:M_0})\} \subset \msp^{+,p}_{1, \tilde \cnst }(\R^{d\times (N_{0}+1)}).$
For the next part, note that \Cref{prop:conv-result}-\ref{item:conv-filtden-like} \& \ref{item:conv-ln-f} (which we can apply because of the first part)   show that
\begin{align*}
	\RE\Big(\filtden_{\drft^{(L)}} & (\cdot| \by_{1:M_0})\big\| \filtden_{\drft^{(L)}_1}(\cdot| \by_{1:M_0})\Big) =\ \ln\lf(\Yden(\by_{1:M_0}|\drft^{(L)}_1)/\Yden(\by_{1:M_0}|\drft^{(L)})\ri)\\
	& +  \int_{\R^{d\times (N_{0}+1)}}  \ln \Xden(\bx_{0:N_0}|\drft^{(L)}) \filtden_{\drft^{(L)}}(d\bx_{0:N_0}| \by_{1:M_0})
	 -  \int_{\R^{d\times (N_{0}+1)}}  \ln \Xden(\bx_{0:N_0}|\drft^{(L)}_1) \filtden_{\drft^{(L)}}(d\bx_{0:N_0}| \by_{1:M_0})\\
\stackrel{L\rt \infty}\Rt &\ \ln\lf(\Yden(\by_{1:M_0}|\drft_1)/\Yden(\by_{1:M_0}|\drft)\ri) +  \int_{\R^{d\times (N_{0}+1)}}  \ln \Xden(\bx_{0:N_0}|\drft) \filtden_{\drft}(d\bx_{0:N_0}| \by_{1:M_0})\\
	&\ -  \int_{\R^{d\times (N_{0}+1)}}  \ln \Xden(\bx_{0:N_0}|\drft_1) \filtden_{\drft}(d\bx_{0:N_0}| \by_{1:M_0})
	=\ \RE\lf(\filtden_{\drft}(\cdot| \by_{1:M_0})\big\| \filtden_{\drft_1}(\cdot| \by_{1:M_0})\ri).	
\end{align*}
\end{proof}

The next corollary, which is a direct consequence of \Cref{lem:ext-unq-min} and the first part of \Cref{cor:conv-RE},  shows the existence of EM-sequence in RKHS, $\Hsp_{\knl}.$
\begin{corollary}\label{cor:exst-em-map}
Suppose that \Cref{assum:sde-knl}:\ref{item-assum-knl}-\ref{item-assum-obsden} hold. Then for any $\drft \in \Hsp_{\knl}$ and $p\geq 0$, $\filtden_{\drft}(\cdot| \by_{1:M_0}) \in \msp^{+,p}_{1}(\R^{d\times (N_{0}+1)})$, and $\Phi(\filtden_{\drft}(\cdot| \by_{1:M_0}) )$ is well-defined in the sense it is the unique (global) minimizer of the function $\risk\Big(\cdot, \filtden_{\drft}(\cdot| \by_{1:M_0})\Big).$
\end{corollary}

\np
The following instructive result lists important properties of $\stset$, the stationary set of $\Phi$ (defined in \eqref{eq:stat-set}), including the fact that any local minimizer of the loss function $\loss$, in particular $\hat \drft_{\mrm{pml}}$, is a member of $\stset$.
\begin{corollary}\label{cor:prop-st-set}
Under \Cref{assum:sde-knl}:\ref{item-assum-knl}-\ref{item-assum-obsden}, the following hold.
\begin{enumerate}[label=(\roman*), ref=(\roman*)]
	\item \label{item:S-closed} $\stset$ is weakly sequentially closed (and hence, also strongly closed).
	\item \label{item:weak-strong-conv} Weak and strong convergences are equivalent on $\stset$. In other words for any sequence $\{\drft^{(L)}_*\} \subset \stset$, $\drft^{(L)}_* \stackrel{w} \Rt \drft_*$  if and only if  $\drft^{(L)}_* \stackrel{s} \Rt \drft_*$ as $L \rt \infty$.
    \item \label{item:loc-min-S} Let $\hat \drft_{\mrm{loc}}$ be a local minimum of the loss function $\loss$. Then  $\hat \drft_{\mrm{loc}} \in \stset$. In particular, $\hat \drft_{\mrm{pml}} \in \stset$.
    \item \label{item:tset-nw-dns} $\stset$ is nowhere dense in the strong topology.
\end{enumerate}	
\end{corollary}

\begin{proof}
	Both \ref{item:S-closed} and \ref{item:weak-strong-conv} are immediate consequences of \Cref{prop:conv-result}-\ref{item:conv-filtden-like}, \Cref{th:M-step-map-cont} and  the definition of $\stset$.
To prove
\ref{item:loc-min-S}, fix a local minimum, $\hat \drft_{\mrm{loc}}$, of $\loss$. There exists an $r_{0}$ such that for all $\drft$ satisfying $\|\drft-\hat \drft_{\mrm{loc}}\|_{\Hsp_{\knl}} \leq r_0$,
%\begin{align*}
$\loss(\hat \drft_{\mrm{loc}}) \leq \loss(\drft).$
%\end{align*}
Now suppose $\hat \drft_{\mrm{loc}} \notin \stset$. Then there exists a $\tilde \drft \neq \hat \drft_{\mrm{loc}}$ such that $\risk\big(\tilde \drft, \filtden_{\hat \drft_{\mrm{loc}}}(\cdot| \by_{1:M_0})\big) < \risk\big(\hat\drft_{\mrm{loc}}, \filtden_{\hat \drft_{\mrm{loc}}}(\cdot| \by_{1:M_0})\big)$. By convexity of the function $\risk\big(\cdot, \filtden_{\hat \drft_{\mrm{loc}}}(\cdot| \by_{1:M_0})\big)$, it is clear that for any $0 < v\leq 1$,  $\drift_v \dfeq v \tilde \drft +(1-v) \hat\drft_{\mrm{loc}}$ satisfies 
$\risk\big(\drift_v, \filtden_{\hat \drft_{\mrm{loc}}}(\cdot| \by_{1:M_0})\big) < \risk\big(\hat\drft_{\mrm{loc}}, \filtden_{\hat \drft_{\mrm{loc}}}(\cdot| \by_{1:M_0})\big).$
%$$\risk\big(\drift_v, \filtden_{\hat \drft_{\mrm{loc}}}(\cdot| \by_{1:M_0})\big) < \risk\big(\hat\drft_{\mrm{loc}}, \filtden_{\hat \drft_{\mrm{loc}}}(\cdot| \by_{1:M_0})\big),\quad \text { for any } 0 < v\leq 1$$
\eqref{eq:loss-meas} then implies that $\loss(\drift_v) < \loss(\hat \drft_{\mrm{loc}})$ for any $0 < v\leq 1$. Since $\drift_v \stackrel{s} \Rt \hat \drft_{\mrm{loc}}$ as  $v \rt 0$, there exists a $0< v_0\leq 1$ such that $\|\drift_{v_0}-\hat \drft_{\mrm{loc}}\|_{\Hsp_{\knl}} \leq r_0$. But then we get $\loss(\hat \drft_{\mrm{loc}}) \leq \loss (\drift_{v_0}) < \loss(\hat \drft_{\mrm{loc}})$, which is a contradiction.
% Since for any $\drft \in \Hsp_{\knl}$, $\loss(\hat \drft_{\mrm{pml}}) \leq \loss(\drft)$ by the definition of $\hat \drft_{\mrm{pml}}$ (c.f. \eqref{eq:lik-min}),  taking $\eta = \filtden_{\hat \drft_{\mrm{pml}}}(\cdot| \by_{1:M_0})$ in \eqref{eq:loss-meas}, we get for any $ \drft \in \Hsp_{\knl}$
%\begin{align*}
%	\risk\lf(\hat \drft_{\mrm{pml}}, \filtden_{\hat \drft_{\mrm{pml}}}(\cdot| \by_{1:M_0})\ri) \leq &\ \risk\lf(\drft,\filtden_{\hat \drft_{\mrm{pml}}}(\cdot| \by_{1:M_0})\ri ) -\RE\lf(\filtden_{\hat \drft_{\mrm{pml}}}(\cdot| \by_{1:M_0})\|\filtden_{\drft}(\cdot| \by_{1:M_0})\ri) \leq \ \risk\lf(\drft,\filtden_{\hat \drft_{\mrm{pml}}}(\cdot| \by_{1:M_0}) \ri),
%\end{align*}
%which shows that $\hat \drft_{\mrm{pml}} =  \Phi\lf(\filtden_{\hat \drft_{\mrm{pml}}}(\cdot| \by_{1:M_0})\ri)  = \argmin_{\drft \in \Hsp_{\knl}} \risk\lf(\drft,\filtden_{\hat \drft_{\mrm{pml}}}(\cdot| \by_{1:M_0})\ri) $;  that is, $\hat \drft_{\mrm{pml}} \in \stset$.

For \ref{item:tset-nw-dns}, note that since $\stset$ is strongly closed by \ref{item:S-closed}, we need to show $\stset^\circ=\emptyset$. This follows from \ref{item:weak-strong-conv} and \Cref{lem:empty-int}.

\end{proof}

\np
As an important consequence of \Cref{prop:conv-result} we have the following  result  saying that the restricted mapping $\Phi:\msp^{+,p}_{1, \cnst}(\R^{d\times (N_{0}+1)}) \Rt \Hsp_{\knl}$ is continuous. The interesting feature is that while $\msp^{+,p}_{1, \cnst}(\R^{d\times (N_{0}+1)})$ is equipped only with the topology of weak convergence, the continuity property holds in  the strong (norm) topology of range space, $\Hsp_{\knl}$.
	
 \begin{theorem}
	\label{th:M-step-map-cont}
 Suppose that \Cref{assum:sde-knl}-\ref{item-assum-knl} \& \ref{item-assum-a} hold and for some $p>2 \kexp\vee1+\aexp$ and $\cnst>0$, $\{\gmeas^{(L)}:\ L\geq 1\} \subset \msp^{+,p}_{1, \cnst}(\R^{d\times (N_{0}+1)}).$ Assume that $\gmeas^{(L)}  \stackrel{L\rt \infty}\RT \gmeas$.
Then as $L\rt \infty$
$$	\drft^{(L)}_*\dfeq \Phi(\gmeas^{(L)}) \stackrel{s}\Rt \drft_* \dfeq \Phi(\gmeas).$$
% the following assertions hold.
%\begin{enumerate}[label=(\roman*), ref=(\roman*)]
%\item \label{item:uniform-tightness} For any sequence $\{\drft^{(L)}\} \subset \Hsp_{\knl}$ satisfying $\drft^{(L)} \stackrel{L\rt \infty}\Rt \tilde\drft \in \Hsp_{\knl}$ pointwise, $\risk(\drft^{(L)}, \gmeas^{(L)}) \rt  \risk(\tilde\drft, \tilde\gmeas).$
%\item $\drft^{(L)}_* \dfeq \Phi(\gmeas^{(L)}) \stackrel{L\rt \infty}\Rt \drft_* \dfeq \Phi(\gmeas).$
%\end{enumerate}

\end{theorem}	

\begin{proof}
%Since $u\in \R^d \Rt |u|^p$ and $u \in \R^d \Rt |\knl(u,u)|$ are l.s.c. functions, and $\gmeas^{(L)}  \stackrel{L\rt \infty}\Rt \gmeas$, it follows by Fatou's lemma that 
%\begin{align*}
%\int_{\R^d} |x_k|^p \gmeas_k(dx_k) \leq \liminf_{L \rt \infty } \int_{\R^d} |x_k|^p \gmeas^{(L)}_k(dx_k) \leq \cnst_p.
%\end{align*}
%Similarly, $\int_{\R^d} \|\knl(x_k,x_k)\|^p \gmeas_k(dx_k) \leq \cnst_p,$ where $\gmeas_k$ denotes the $k$-th marginal of $\gmeas$.
%
 Since $\ln \Xden(\bx_{0:N_0}|\drft) \leq 0$, \eqref{eq:risk-funct} and   \Cref{prop:conv-result}-\ref{item:conv-ln-f} show that for any $\drft \in \Hsp_{\knl}$, 
\begin{align}\label{eq:drift-seq-bd}
\l \|\drft^{(L)}_*\|_{\Hsp_{\knl}} \leq \risk\lf(\drft^{(L)}_*, \gmeas^{(L)}\ri) \leq \risk(\drft, \gmeas^{(L)}) \stackrel{L\rt \infty}\Rt\risk(\drft, \gmeas).
\end{align}
Hence the family $\{\|\drft^{(L)}_*\|_{\Hsp_{\knl}}\}$ is bounded, and it follows by Banach-Alaoglu theorem that $\{\drft^{(L)}_*\}$ is weakly compact in $\Hsp_{\knl}$. Therefore, there exists a subsequence $\{L_j\}$ such that $\drft^{(L_j)}_* \stackrel{w} \Rt \drft^{(0)}$ as $j\rt \infty.$ We show that $\drft^{(0)} = \drft_*$, and thus the limit point is independent of the subsequence. Consequently,  $\{\drft^{(L)}_*\}$ converges weakly to $\drft_*$ along the full sequence, that is, as $L\rt \infty$
\begin{align}
	\label{eq:drift-seq-weak-conv}
	\drft^{(L)}_* \stackrel{w}\Rt \drft_*
\end{align}

%  Toward this end, notice that the the above weak convergence of $\drft^{(L_j)}_*$ to $\drft^{(0)} $ implies $\drft^{(L_j)}_* \stackrel{j\rt \infty}\Rt \drft^{(0)}$ pointwise; indeed  for any $x \in \R^d$,
%$$\drft^{(L_j)}_*(x) = \<\knl(x,\cdot), \drft^{(L_j)}_*\> \stackrel{j \rt \infty}\Rt \<\knl(x,\cdot), \drft^0\> = \drft^0(x). $$
We now work toward establishing this. Since the mapping $ \drft \in \Hsp_{\knl} \rt \l \|\drft\|_{\Hsp_{\knl}} $ is weakly l.s.c, we have that $\liminf_{j\rt \infty } \|\drft^{(L_j)}_*\|_{\Hsp_{\knl}} \geq \|\drft^{0}\|_{\Hsp_{\knl}}.$
This and \Cref{prop:conv-result}-\ref{item:conv-ln-f}  now show that  
\begin{align*}
\liminf_{j\rt \infty}	 \risk(\drft^{(L_j)}_*, \gmeas^{(L_j)})   
	 =&\ \liminf_{j\rt \infty}\lf[ -\int_{\R^{d\times (N_{0}+1)}}   \ln \Xden(\bx_{0:N_0}|\drft^{(L_j)}_*) \gmeas^{(L_j)}(d\bx_{0:N_0})+ \l\|\drft^{(L_j)}_*\|^2_{\Hsp_{\knl}} \ri]\\
	 \geq&\ - \int_{\R^{d\times (N_{0}+1)}}  \ln \Xden(\bx_{0:N_0}|\drft^{(0)})\gmeas(d\bx_{0:N_0})+\l \|\drft^{(0)}\|^2_{\Hsp_{\knl}} =  \risk(\drft^{(0)}, \gmeas).
\end{align*}
The above inequality and \eqref{eq:drift-seq-bd} imply that for any $\drft \in \Hsp_{\knl}$
\begin{align}
	\label{eq:risk-conv}
 \risk(\drft^{(0)}, \gmeas) \leq \liminf_{j\rt \infty} \risk(\drft^{(L_j)}_*, \gmeas^{(L_j)}) \leq \limsup_{j\rt \infty} \risk(\drft^{(L_j)}_*, \gmeas^{(L_j)})  \leq \risk(\drft,\gmeas).	
\end{align}
Consequently, $\drft^{(0)} = \drft_*\equiv \Phi(\gmeas)= \argmin\limits_{\drft \in \Hsp_{\knl}} \risk(\drft, \gmeas)$ which establishes \eqref{eq:drift-seq-weak-conv}. To establish the strong convergence of  $\{\drft^{(L)}_*\}$ to $\drft_*$, notice that taking $\drft=\drft_*$ in \eqref{eq:risk-conv} shows that $\lim_{L \rt \infty} \risk(\drft^{(L)}_*, \gmeas^{(L)}) =\risk(\drft_*,\gmeas).$ Because of \eqref{eq:risk-funct} and  \Cref{prop:conv-result}-\ref{item:conv-ln-f}, we then must have $\lim_{L \rt \infty }\|\drft^{(L)}_*\|_{\Hsp_{\knl}} =  \|\drft_*\|_{\Hsp_{\knl}} $, which together with \eqref{eq:drift-seq-weak-conv} establishes  $	\drft^{(L)}_* \stackrel{s}\Rt \drft_*$ as $L\rt \infty$.
\end{proof}

We are now ready state the primary result of this section summarizing the properties of the approximate EM sequence.

\begin{theorem}
\label{th:AEM-prop}
Suppose that \Cref{assum:sde-knl} hold. For each $\drft \in \Hsp_{\knl}$, let $\big\{\hat \filtden^{(L)}_{\drft}(\cdot| \by_{1:M_0})\big\}$ be such that 
$\RE\Big(\hat \filtden^{(L)}_{\drft}(\cdot| \by_{1:M_0})\|\filtden_{\drft}(\cdot| \by_{1:M_0})\Big) \stackrel{L \rt \infty} \Rt 0.$ Then for any initial $\drft_0 \in \Hsp_{\knl}$,  there exists $L_{k-1}>0$ for $k$-th iteration such that the approximate EM sequence $\big\{\hat\drft_k\equiv \hat\drft_k(\drft_0, \by_{1:M_0}):k\geq 0\big\}$, defined by \eqref{eq:EM-approx}, has the following properties.
\begin{enumerate}[label=(\roman*), ref=(\roman*), resume=app-em-list]
	\item \label{item:app-loss-seq-conv} $\loss(\hat \drft_k) \stackrel{k\rt \infty} \Rt  \l_{\drft_0}$ for some $ \l_{\drft_0} \in \R.$ 
	\item \label{item:app-em-cmpt-stat} $\{\hat \drft_k, k\geq 0\}$ is  strongly pre-compact, and the set of its limit points, $\hat \lmtset(\drft_0) \subset \stset$.
	\item \label{eq:app-limit-set-form} $\hat {\lmtset}(\drft_0) \subset \loss^{-1}(\l_{\drft_0}) \equiv \big\{\tilde \drft \in \Hsp_{\knl} : \loss(\tilde \drft) = \l_{\drft_0}\big\}$, where $ \l_{\drft_0}$ is as in \ref{item:app-loss-seq-conv}.
	\item \label{eq:app-limit-set-conn} $\hat{\lmtset}(\drft_0)$ is  connected in $\Hsp_{\knl}$ (with respect to strong topology).
\end{enumerate}
%Furthermore, the set of stationary points $\stset$ has the following decomposition
%\begin{enumerate}[label=(\roman*), ref=(\roman*), resume=app-em-list]	
%         \item \label{eq:app-limit-set-empt-int} $\stset = \cup_{\drft_0} \hat{\lmtset}(\drft_0)= \hat \lmtset_0 \cup \hat \lmtset_1$, where $\hat \lmtset_0$ is the set of $\hat \drft_* \in \stset$ such that any approximate EM sequence having $\hat \drft_*$ as a limit point is eventually constant and $\hat \lmtset_1 \equiv \stset-\lmtset_0$ has empty interior. In particular, for any $\drft_0 \in \Hsp_{\knl},$ $\hat {\lmtset}(\drft_0)$ is nowhere dense.
%\end{enumerate}

\end{theorem}	

\begin{proof}
\ref{item:app-loss-seq-conv} Let $\{\vep_k:k\geq 1\}$ be such that $E \dfeq \sum_{k\geq 0} \vep_k < \infty.$ By the hypothesis and Pinsker's inequality, for each $k\geq 1$, choose $L_{k-1}$  such that 
\begin{align}\label{eq:relent-est}
2\lf\|\hat \filtden^{(L_{k-1})}_{\hat \drft_{k-1}}(\cdot| \by_{1:M_0})-\filtden_{\hat \drft_{k-1}}(\cdot| \by_{1:M_0})\ri\|_{\mrm{TV}}^2\leq \RE\Big(\hat \filtden^{(L_{k-1})}_{\hat \drft_{k-1}}(\cdot| \by_{1:M_0})\| \filtden_{\hat \drft_{k-1}}(\cdot| \by_{1:M_0})\Big) \leq \vep_k.
\end{align}
Now the inequality, which holds by definition of $\risk$, \eqref{eq:risk-funct}, and approximate EM-sequence, \eqref{eq:EM-approx},
 $$\risk\Big(\hat\drft_k, \hat \filtden^{(L_{k-1})}_{\hat \drft_{k-1}}(\cdot| \by_{1:M_0})\Big) \leq \risk\Big(\hat\drft_{k-1}, \hat \filtden^{(L_{k-1})}_{\hat \drft_{k-1}}(\cdot| \by_{1:M_0})\Big),$$
implies that (by \eqref{eq:loss-meas})
\begin{align}\label{eq:app-loss-ineq}
\loss(\hat \drft_k) +\RE\lf(\hat \filtden^{(L_{k-1})}_{\hat \drft_{k-1}}(\cdot| \by_{1:M_0})\|\filtden_{\hat b_k}(\cdot| \by_{1:M_0})\ri) \leq \loss(\hat \drft_{k-1})+\RE\Big(\hat \filtden^{(L_{k-1})}_{\hat \drft_{k-1}}(\cdot| \by_{1:M_0})\| \filtden_{\hat \drft_{k-1}}(\cdot| \by_{1:M_0})\Big) 
\end{align}
giving
\begin{align}\label{eq:app-loss-diff}
\loss(\hat \drft_k) - \loss(\hat \drft_{k-1}) \leq \vep_k.
\end{align}
Summing over $k=1,2,\hdots \tilde k$, we get for any $\tilde k$, $\loss(\hat \drft_{\tilde k})  \leq E + \loss(\drft_0).$  \\ Now clearly, $\ell(b|\by_{1:M_0}) = \ln \Yden(\by_{1:M_0}|\drft)  \leq  M_0 \ln\obsdenbd$  giving
\begin{align}
\label{eq:like-bd}
		 \loss(\hat \drft_k) \geq -M_0 \ln\obsdenbd+  \l \|\hat \drft_k\|^2_{\Hsp_{\knl}} \geq  - M_0 \ln\obsdenbd;
\end{align}
Thus the sequence $\{\loss(\hat \drft_k): k\geq 0 \}$ is bounded, and hence $\l_{\drft_0} \dfeq \liminf_{k\rt \infty} \loss(\hat\drft_k) \in \R$. Fix a $\delta>0$. Choose a $K_0$ such that $ \loss(\hat \drft_{K_0}) \leq \l_{\drft_0} +\delta/2$ and  $\sum_{k> K_0}\vep_k \leq \delta/2.$
%\begin{align*}
%\hat\l_{\drft_0}  \leq \loss(\hat \drft_{K_0}) \leq \hat\l_{\drft_0} +\delta/2, \quad \sum_{k\geq K_0}\vep_k \leq \delta/2.
%\end{align*}
Then for any $k' \geq K_0$, summing \eqref{eq:app-loss-diff} from $K_0+1$ to $k'$ we get
 $\loss(\hat \drft_{k'}) \leq \loss(\hat \drft_{K_0}) +\sum_{k> K_0}\vep_k \leq  \l_{\drft_0}+\delta$, which shows that 
 $$\limsup_{k\rt \infty}\loss(\drft_k) \leq  \l_{\drft_0}+\delta.$$
  Since $\delta>0$ is arbitrary, we conclude 
 $$\l_{\drft_0} = \liminf_{k\rt \infty} \loss(\hat \drft_k) \leq \limsup_{k\rt \infty}\loss(\hat \drft_k) \leq \l_{\drft_0}$$
 proving  \ref{item:app-loss-seq-conv}.
 
 \np
\ref{item:app-em-cmpt-stat} We first show that $\{\hat \drft_k, k\geq 0\}$ is  weakly pre-compact in $ \Hsp_{\knl}$.  Notice that \eqref{eq:like-bd} gives
\begin{align*}
 \l \|\hat \drft_k\|^2_{\Hsp_{\knl}} \leq \loss(\hat \drft_k) +M_0 \ln\obsdenbd \stackrel{k\rt \infty}\Rt \l_{\drft_0}+M_0 \ln\obsdenbd.
\end{align*}
Hence the sequence $\{\hat \drft_k\}$ is norm-bounded in $\Hsp_{\knl}$ and thus weakly compact by Banach-Alaoglu theorem. To prove $\hat \lmtset(\drft_0) \subset \stset$, let $\hat b_* \in \hat \lmtset(\drft_0)$ be a weak-limit point of of $\{\hat \drft_k\}$, and $\{\hat \drft_{k'_j} \}$  a subsequence such that as $j\rt \infty$, $\hat \drft_{k'_j} \stackrel{w}\Rt b_*.$
Note that there exists a further subsequence $\{k_j\} \subset \{k'_j\}$ such that as $j \rt \infty$
\begin{align}\label{eq:app-em-subseq-conv-0}
\hat \drft_{k_j-1} \stackrel{w}\Rt \hat \drft_{*,1}, \quad \hat \drft_{k_j} \stackrel{w}\Rt \hat \drft_*.
\end{align}
for some $\hat \drft_{*,1} \in \hat \lmtset(\drft_0)$. 
\Cref{prop:conv-result}-\ref{item:conv-filtden-like} then implies that 
\begin{align}\label{eq:conv-filt}
\filtden_{\hat\drft_{k_j-1}}(\cdot| \by_{1:M_0}) \Rt \filtden_{\hat\drft_{*,1}}(\cdot| \by_{1:M_0}), \quad \filtden_{\hat \drft_{k_j}}(\cdot| \by_{1:M_0}) \Rt \filtden_{\hat\drft_*}(\cdot| \by_{1:M_0})
\end{align}
pointwise and in $L^{1}(\R^{d\times (N_{0}+1)})$ (equivalently, in total variation). Consequently, by \eqref{eq:relent-est},
\begin{align}\label{eq:app-filt-est}
\lf\|\hat \filtden^{(L_{k_j-1})}_{\hat \drft_{k_j-1}}(\cdot| \by_{1:M_0}) - \filtden_{\hat\drft_{*,1}}(\cdot| \by_{1:M_0})\ri\|_{\mrm{TV}} \leq \lf\|\filtden_{\hat\drft_{k'_j-1}}(\cdot| \by_{1:M_0})-\filtden_{\hat\drft_{*,1}}(\cdot| \by_{1:M_0})\ri\|_{\mrm{TV}}+ \sqrt{\vep_{k_j}/2} \stackrel{j\rt \infty} \Rt 0. 
\end{align}
%It follows from \eqref{eq:rel-ent-conv-0} and the joint lower-semicontinuity of the mapping $(\gmeas_1,\gmeas_2) \in \msp^+_{1}(\R^{d\times (N_{0}+1)}) \rt \RE(\gmeas_1\|\gmeas_2) \in [0,\infty)$ that
Pinsker's inequality,  \eqref{eq:relent-est}, \eqref{eq:app-loss-ineq} and \ref{item:app-loss-seq-conv} show that 
\begin{align*}%\label{eq:rel-ent-conv-0}
2\lf\|\hat \filtden^{(L_{k_j-1})}_{\hat \drft_{k_j-1}}(\cdot| \by_{1:M_0})-\filtden_{\hat b_{k_j}}(\cdot| \by_{1:M_0})\ri\|^2_{\mrm{TV}}	\leq&\ \RE\lf(\hat \filtden^{(L_{k_j-1})}_{\hat \drft_{k_j-1}}(\cdot| \by_{1:M_0})\|\filtden_{\hat b_{k_j}}(\cdot| \by_{1:M_0})\ri) \\
\leq&\ \loss(\hat \drft_{k_j-1}) - \loss(\hat \drft_{k_j})+\vep_{k_j}\ \stackrel{j\rt \infty} \Rt \ \l_{\drft_0} -\l_{\drft_0} =0.
\end{align*}
On the other hand,  \eqref{eq:app-filt-est} and \eqref{eq:conv-filt} show that
$$\lf\|\hat \filtden^{(L_{k_j-1})}_{\hat \drft_{k_j-1}}(\cdot| \by_{1:M_0})-\filtden_{\hat b_{k_j}}(\cdot| \by_{1:M_0})\ri\|_{\mrm{TV}} \stackrel{j\rt \infty} \Rt \lf\|\filtden_{\hat\drft_{*,1}}(\cdot| \by_{1:M_0}) - \filtden_{\hat\drft_*}(\cdot| \by_{1:M_0})\ri\|_{\mrm{TV}}.$$
Therefore 
\begin{align*}
\lf\|\filtden_{\hat\drft_{*,1}}(\cdot| \by_{1:M_0}) - \filtden_{\hat\drft_*}(\cdot| \by_{1:M_0})\ri\|_{\mrm{TV}}=0.
\end{align*}
Thus $\filtden_{\hat \drft_*}(\cdot| \by_{1:M_0}) = \filtden_{\hat \drft_{*,1}}(\cdot| \by_{1:M_0})$, and because of \Cref{lem:eqv-ident} and \Cref{assum:sde-knl}-\ref{item:ident}, $\hat \drft_* = \hat\drft_{*,1}.$
Now by \eqref{eq:app-em-subseq-conv-0} and \Cref{th:M-step-map-cont}, we have as $j\rt \infty$
\begin{align*}
\hat \drft_{*}\stackrel{w}\Lt\hat \drft_{k_j}\equiv \Phi\Big(\hat \filtden^{(L_{k_j-1})}_{\hat \drft_{k_j-1}}(\cdot| \by_{1:M_0})\Big) \stackrel{s}\Rt \Phi\big(\filtden_{\hat \drft_{*,1}}(\cdot| \by_{1:M_0})\big).
\end{align*}
Since $\hat \drft_* = \hat \drft_{*,1},$ it follows that $\hat \drft_* = \Phi\big(\filtden_{\hat \drft_{*}}(\cdot| \by_{1:M_0})\big)$, that is, $\hat\drft_* \in \SC{S},$ proving that  $\hat\lmtset(\drft_0) \subset \stset$.  Furthermore, this also shows that $\hat \drft_{k_j} \stackrel{s}\Rt  \hat \drft_{*}$ as $j\rt \infty$ proving that 	$\{\hat \drft_k, k\geq 0\}$ is strongly pre-compact.

\np
\ref{eq:app-limit-set-form} We need to show that $\loss(\hat \drft_*) =\l_{\drft_0}$. By strong convergence of $\hat \drft_{k_j}$ to $\hat \drft_*$, $\|\hat \drft_{k_j}\|_{\Hsp_{\knl}} \stackrel{j\rt \infty} \Rt \|\hat \drft_{*}\|_{\Hsp_{\knl}}$. This and  \Cref{prop:conv-result}-\ref{item:conv-filtden-like} now show
\begin{align*}
\l_{\drft_0} = \lim_{j\rt \infty} \loss(\hat\drft_{k_j})  = -\lim_{j\rt \infty} \ell(\hat\drft_{k_j}|\by_{1:M_0})+ \lim_{j\rt \infty}\l\|\hat \drft_{k_j}\|^2_{\Hsp_{\knl}} = -\ell(\hat \drft_{*}|\by_{1:M_0})+ \l\|\hat \drft_{*}\|^2_{\Hsp_{\knl}} \equiv \loss(\hat \drft_{*}).
\end{align*}
\np
\ref{eq:app-limit-set-conn} By \ref{item:app-em-cmpt-stat},  $\overline{\{\hat\drft_k, k\geq 0\}}$ is a strongly compact subset of $\Hsp_{\knl}$, and $\hat\lmtset(\drft_0) \subset \overline{\{\hat\drft_k, k\geq 0\}}$. Next notice that the previous findings show that for every subsequence $\{\tilde k_j\}$, there exists a further subsequence $\{k_j\}$ such that  as $j \rt \infty,$
 $\hat\drft_{k_j} -\hat \drft_{k_j-1} \stackrel{s}\Rt 0.$ We then conclude that $\hat \drft_{k} -\hat \drft_{k-1} \stackrel{s}\Rt 0$ as $k\rt \infty$. Consequently, the assertion \ref{eq:app-limit-set-conn} follows by  \cite[Theorem 1]{AsAd70}. 
 
% \np
% \ref{eq:app-limit-set-empt-int} The first equality, $\stset = \cup_{\drft_0} \hat{\lmtset}(\drft_0)$, is immediate from \ref{item:app-em-cmpt-stat} and the definition of $\stset$ in \eqref{eq:stat-set}. Now suppose that $(\hat \lmtset_1)^0 \neq \emptyset.$ Then there exists an $\hat \drft_* \in \hat \lmtset_1$ and an $r_0>0$ such that the open ball $B(\hat \drft_*,r_0) \subset \hat \lmtset_1$. By the definition of $\hat \lmtset_1$ there exists an approximate EM sequence $\big\{\hat\drft_k\equiv \hat\drft_k(\drft_0, \by_{1:M_0}):k\geq 0\big\}$ (starting from some initial $\drft_0$) which is not eventually constant and which has $\hat \drft_*$ as its limit point. There is a $k_0$ such that $\hat\drft_{k_0} \in B(\hat \drft_*,r_0) \subset \hat \lmtset_1 \subset \stset$. By the definition of $\stset$, this means that  $\hat\drft_{k} = \hat\drft_{k_0}$ for all $k \geq 0$ contradicting the assumption that $\{\hat\drft_k\}$ is  not eventually constant. Therefore, $(\hat \lmtset_1)^0 = \emptyset.$ 
% 
% For the assertion on the limit set $\hat {\lmtset}(\drft_0)$, notice that for any $\drft_0 \in \Hsp_{\knl},$ $\hat {\lmtset}(\drft_0)$ is either a singleton or a subset of $\hat \lmtset_1$.
\end{proof}

The following result on the original EM sequence is just a special case of \Cref{th:AEM-prop} obtained by simply taking $\hat \filtden^{(L)}_{\drft}(\cdot| \by_{1:M_0}) \equiv \filtden_{\drft}(\cdot| \by_{1:M_0})$.
\begin{corollary}\label{cor:org-EM}
Under \Cref{assum:sde-knl}, the original EM sequence $\{\drft_k\}$ defined by \eqref{eq:EM-0}  satisfies \ref{item:app-loss-seq-conv} - \ref{eq:app-limit-set-conn} of \Cref{th:AEM-prop} with the additional property that $\{\loss(\drft_k)\}$ is decreasing.
\end{corollary}

Finding the precise $L\equiv L_{k-1}$ for $k$-th iteration is typically not possible, unless the rate of convergence of $\hat \filtden^{(L)}_{\drft}(\cdot| \by_{1:M_0})$ to $ \filtden_{\drft}(\cdot| \by_{1:M_0})$ in KL-divergence is known. The importance of \Cref{th:AEM-prop} lies in showing that, by choosing a large enough $L$, it is possible to construct an approximate EM sequence that closely follows the original and retains its desirable convergence properties.
Note that $\{\hat \drft_k\}$ converges along the full sequence if any of the following  conditions hold (i) $\stset$ is a singleton,  (ii) the loss function $\loss:\Hsp_{\knl} \rt \R$ is injective, or (iii) the limit point-set $\hat\lmtset(\drft_0)$ is discrete (as a discrete connected set must be singleton).

%%%% Original result on EM sequence after end{document}

\subsection{SMC and Implementation of EM-algorithm}\label{sec:smc-em}
As mentioned in Section~\ref{sec:EM}, a tractable approximation of the filtering distribution $\filtden_{\drft}(\cdot|\by_{1:M_0})$ is essential for implementing the EM algorithm. This can be achieved by constructing a Monte Carlo estimate through approximate samples drawn from $\filtden_{\drft}(\cdot|\by_{1:M_0})$. Since, for any $\drft$, the density $\filtden_{\drft}(\cdot|\by_{1:M_0})$ (see~\eqref{eq:filtden-exp}) is known only up to a normalizing constant, importance sampling or MCMC methods can be used to generate such samples. In the special case where the observations $\by_{1:M_0}$ are noise-free measurements of $X$ at discrete time points ${t_m}$, this reduces to sampling from diffusion bridges, for which there is a substantial body of theoretical and computational work, e.g., \cite{LyZh90, DeHu06, DuGa02, BeRoSt08, Lind12, PaRo12, BlSo14, ScMeZa17, WGBS17-2}. In particular, the proposal of Durham and Gallant~\cite{DuGa02} (see also its variant by Lindstr\"om~\cite{Lind12}), based on a zeroth-order approximation of $\drft$, is widely used and easy to implement. This proposal was later adapted in \cite{GoWi08} (see also \cite{WGBS17-2}) to accommodate partial and noisy observations within an MCMC framework for simulating approximate samples from the filtering distribution. When the number of missing points between two successive observation time points, that is, $n_{m} -n_{m-1}$, is large, the mixing rate of MCMC is very low even with block updating technique, which becomes particularly problematic for our EM algorithm.

  In this paper, we approximate the filtering distribution $\filtden_{\drft} (\cdot|\by_{1:M_0})$  using SMC. The structure of the proposal enables systematic construction of particle-paths by sequentially drawing latent states at times $s_{n}$ and recursively updating importance weights at each observation time $t_m\equiv s_{n_m}$. This recursive formulation makes SMC significantly more efficient and scalable than MCMC in our setting, particularly as the time horizon or dimensionality of the system increases. Moreover, resampling mitigates particle degeneracy, which is a common limitation of basic importance sampling schemes in high-dimensional or long time-horizon settings. Previous works on SMC for SDE models include \cite{FePaRo08}, which introduced a random-weight particle filter for a class of SDE models by using unbiased estimation of transition densities rather than discretization-based approximations, and   \cite{LiChMy10}, which simulated diffusion bridges (the case of noise-free observations) using Durham and Gallant’s proposal  with resampling strategies guided by backward pilots and priority scores; see also \cite{LiCh98}.
   
The {\em modified diffusion bridge} proposal of  Durham and Gallant, which actually does not depend on the drift $\drft$ (see \eqref{eq:DG-prop}), performs well only when the drift function $\drft$ is roughly a constant.  In our highly nonlinear setting, this limitation becomes especially pronounced within the SMC steps of the EM algorithm, where $\drft$ at each iteration is represented as a finite-kernel expansion (as we will see shortly). A linear SDE approximation is one way to design better proposals, and in this paper we employ it in a somwhat different way than is typically done in the literature.
%To overcome this issue, we introduce a carefully designed proposal distribution $\propdist$ based on a first-order approximation of $\drft$, which better captures the local behavior of the drift function. 
The details about the SMC algorithm including the choice of proposal $\propdist$ used in the paper is discussed in Section \ref{sec:SMC}. First, however, we introduce the SMC-EM sequence and explain how the SMC approximation of $\filtden_{\drft}(\cdot| \by_{1:M_0})$ renders the M-step optimization problem tractable via the Representer Theorem  leading to a finite-sum kernel expansion of $\drft$ at each iteration.
  
  Given $\drft$ from the $(k-1)$-th iteration, the SMC-approximation of the probability measure $\filtden_{\drft}(\cdot| \by_{1:M_0})$ on $\R^{d\times N_{0}}$ needed for the E-step of $k$-th iteration is given by
 \begin{align}\label{eq:SMC-filter}
 \hat\filtdist^{\mathrm{SMC},L}_{\drft}(\cdot| \by_{1:M_0}) = \sum_{l=1}^{L} w^{(l,k-1)}_{T} \delta_{\tilde\BX^{(l,k-1)}_{[0:T]}}.
  \end{align}
  where $\tilde\BX^{(l,k-1)}_{[0:T]}, l=1,2,\hdots, L$ are particle-paths on the interval $[0,T=t_{M_0}]$ drawn from a proposal distribution $\propdist$, $w^{(l,k-1)}_T$ is the normalized weight associated with the $l$-th path in $k$-th iteration of EM algorithm satisfying $\sum_{l=1}^L w^{(l,k-1)}_T =1$. 
  
  Starting with initial $\drft_0$,  replacing the filtering distribution by an $L$-particle based SMC-approximation of the form \eqref{eq:SMC-filter} at every iteration  leads to estimation of $\drft$ by the SMC-EM sequence\\  $\lf\{\drft^{(L)}_{k}\equiv \drft^{(L)}_k(\drft_0, \by_{1:M_0}):k\geq 0\ri\}$, defined by 
\begin{align}\label{eq:EM-1}
\begin{aligned}
	\drft^{(L)}_{k} \dfeq &\ \Phi\lf(\hat\filtden^{\mathrm{SMC, L}}_{\drft^{(L)}_{k-1} }(\cdot| \by_{1:M_0})\ri)  = \argmin_{\drft \in \Hsp_{\knl}} \risk\lf(\drft,\hat\filtden^{\mathrm{SMC, L}}_{\drft^{(L)}_{k-1}}(\cdot| \by_{1:M_0})\ri)\\
	=&\ \argmin_{\drft \in \Hsp_{\knl} }\lf[ - \int_{\R^{d\times (N_{0}+1)}} \ln \Xden(\bx_{0:N_0}|\drft) \hat\filtdist^{\mathrm{SMC, L}}_{\drft_{k-1}^{(L)}}(d\bx_{0:N_0}| \by_{1:M_0})   + \l \|\drft\|^2_{\Hsp_{\knl}}\ri] .
\end{aligned}
\end{align}
The convergence properties of SMC methods, including a CLT, are well established; see \cite{CrDo02, Chop04, Del04}. While the probability measure $\hat\filtdist^{\mathrm{SMC, L}}_{\drft}(\cdot \mid \by_{1:M_0})$ does not converge to $\filtdist_{\drft}(\cdot \mid \by_{1:M_0})$ in the KL-divergence --- meaning that \Cref{th:AEM-prop} technically cannot be directly applied to the SMC-EM sequence --- a corresponding kernel density estimate (KDE), for  a suitable smoothing kernel, often does.
 In practice, however, our test runs showed that using a KDE-based version of \eqref{eq:SMC-filter} did not lead to any noticeable improvement in estimating the drift function $\drft$. Therefore, for convenience, we proceed with the SMC-EM sequence as defined in \eqref{eq:EM-1}.

% The following general theorem guaranteed that the minimization problem in  \eqref{eq:EM-1} is indeed a valid approximation of that in \eqref{eq:EM-0} in the sense that for any fixed iteration $k$, $\drft^{(L)}_{k} \stackrel{L\rt \infty}\Rt \drft_{k}.$

The  interesting outcome  is that RKHS theory shows that the complex infinite-dimensional optimization problem in \eqref{eq:EM-1} (M-step)  actually admits a concrete, closed-form solution.
The following theorem provides the details and plays a central role in our estimation procedure.

\begin{theorem}\label{th:fin-sum-rep}
	Consider the minimization problem in \eqref{eq:EM-1}. Define $LdN_0$-dimensional  vectors $\bm{\vart}^{(k-1)}$,  and $Ld(N_0+1)\times Ld(N_0+1)$ matrices $ \bm{\SC{K}_0}^{(k-1)}$ and $\sdg^{(k-1)}$  as 
\begin{align} 
	\begin{aligned}
	\bm{\vart}^{(k-1)} \dfeq &\ \lf(\big(\vart^{(l,k-1)}\big)^\top_{n} = \big(\tilde X^{(l,k-1)}(s_{n})-  \tilde X^l(s_{n-1})\big)^\top: l=1,\hdots, L, n=0, \hdots, N_0\ri)^\top\\
	 	\bm{\SC{K}}^{(k-1)}_0 \dfeq &\ \lf(\lf(\knl\big(\tilde X^{(l,k-1)}(s_n), \tilde X^{(l',k-1)}(s_{n'})\big)\ri)\ri)_{\substack{l=1,\hdots, L\\ n,n'=0,\hdots, N_0}}, \\
	\sdg^{(k-1)} \dfeq &\ \mathrm{diag}\lf(w^{(l,k-1)}_{T}\lf(a(\tilde X^{(l,k-1)}(s_n))\ri)^{-1}\ri)_{\substack{l=1,\hdots, L\\ n=0,\hdots, N_0}}
	\end{aligned} \label{eq:mat-forms}
\end{align}  
where $\tilde X^l(s_{-1}) \equiv 0$.
Then the minimizer $b_k$ is given by
 \begin{align}\label{eq:b-knl-exp}
 b_k(x) = \sum_{l=1}^L\sum_{n=0}^{N_0} \knl\big(x, \tilde X^{(l,k-1)}(s_n)\big) \wtb^{(l,k, *)}_n, \quad \wtb^{(l,k,*)}_n \in \R^{d},
 \end{align}
 where $ \bm{\wtb}^{(k,*)} \dfeq\ \lf (\big(\wtb^{(l,k,*)}_n\big)^\top: l=1,\hdots, L, n=0, \hdots, N_0\ri)^\top$ is given by
\begin{align}\label{eq:opt-wt}
\bm{\wtb}^{(k,*)} = \bcm^{(k-1)}\bm{\SC{K}_0}^{(k-1)} \sdg^{(k-1)} \bm{\vartheta}^{(k-1)}, \quad
\big(\bcm^{(k-1)}\big)^{-1} = \Delta  \big(\bm{\SC{K}_0}^{(k-1)}\big)^\top \sdg^{(k-1)} \bm{\SC{K}_0}^{(k-1)}+\l  \bm{\SC{K}_0}^{(k-1)}.
\end{align}
\end{theorem}

\begin{proof}
	Notice that the optimization problem in \eqref{eq:EM-1} can be rewritten as
\begin{align}\label{eq:EM-2}
 b_k =& \argmin_{b \in \Hsp_{\knl} }\lf[- \sum_{l=1}^L w^{(l,k-1)}_{T}\ln \Xden\lf(\tilde\BX^{(l,k-1)}_{[0,T]}\big|\drft,\dffun\ri) + \l \|b\|^2_{\Hsp_{\knl}}\ri].
\end{align}
Now notice that	
\begin{align*}
-\ln \Xden(\bx_{0:N_0}|\drft) =&\ \Delta^{-1}\sum_{n=1}^{N_0}\lf(x_n - x_{n-1}-\drft(x_{n-1})\Delta\ri)^\top a^{-1}(x_{n-1})\lf(x_n - x_{n-1}-\drft(x_{n-1})\Delta\ri)\\
=&\ \Delta^{-1}\sum_{n=1}^{N_0}(x_n-x_{n-1})^\top a^{-1}(x_{n-1})(x_n-x_{n-1})-2\Delta(x_n-x_{n-1})^\top a^{-1}(x_{n-1})\drft(x_{n-1})\\
&\ +\Delta^2 b^\top(x_{n-1}) a^{-1}(x_{n-1})\drft(x_{n-1}).
\end{align*}
Since the first term in the above summation does not depend on the unknown function $\drft$,
\begin{align*}
 b_k =& \argmin_{b \in \Hsp_{\knl} } \loss\big(b\big|\BX^{(l,k-1)}_{[0:T]}, l=1,\hdots,L\big),
\end{align*}
where the loss functional, $ \loss(b) \equiv \loss\big(b\big|\BX^{(l,k-1)}_{[0:T]}, l=1,\hdots,L\big)$, is given by 
\begin{align}\label{eq:EM-loss}
\begin{aligned}
\loss(b) =&\ \sum_{l=1}^L\sum_{n=1}^{N_0} w^{(l,k-1)}_{T}\Big(\Delta b^\top(\tilde X^{(l,k-1)}(s_{n-1})) a^{-1}(\tilde X^{(l,k-1)}(s_{n-1}))\drft(\tilde X^{(l,k-1)}(s_{n-1})) \\
 & \hs{.4in}-2\lf(\tilde X^{(l,k-1)}(s_n)-\tilde X^{(l,k-1)}(s_{n-1})\ri)^\top a^{-1}(\tilde X^{(l,k-1)}(s_{n-1}))\drft(\tilde X^{(l,k-1)}(s_{n-1}))\Big) + \l \|b\|^2_{\Hsp_{\knl}}.
\end{aligned}
\end{align}
By the Representer Theorem for vector-valued functions \cite[Theorem 15]{GaMiZh23} (see also \cite{MiPo05}), the minimizer $b_k$ admits the expression
 \begin{align}\label{eq:b-exp}
 b_k(u) = \sum_{l=1}^L\sum_{n=0}^{N_0}\knl\big(u, \tilde X^{(l,k-1)}(s_n)\big) \wtb^{(l,k, \mathrm{min})}_n, \quad \wtb^{(l,k,\mathrm{min})}_n \in \R^{d},
 \end{align}
 Now for any $b$ of the form $\drft(u) = \sum_{l=1}^L\sum_{n=1}^{N_0} \knl\big(u, \tilde X^{(l,k-1)}(s_n)\big) \wtb^{(l)}_n$
\begin{align*}
\loss(b) = &\  \Delta \bm{\wtb}^\top \big(\bm{\SC{K}}^{(k-1)}_0\big)^\top \sdg^{(k-1)} \bm{\SC{K}}^{(k-1)}_0\bm{\wtb} - 2  \bm{\vartheta}^\top \sdg^{(k-1)} \bm{\SC{K}}^{(k-1)}_0\bm{\wtb} + \l \bm{\wtb}^\top \bm{\SC{K}}^{(k-1)}_0\bm{\wtb}\\
=&\ - (\bm{\wtb}^{(k,*)})^\top \big(\bcm^{(k-1)}\big)^{-1} \bm{\wtb}^{(k,*)} + \lf(\bm{\wtb}-\bm{\wtb}^{(k,*)}\ri)^\top \big(\bcm^{(k-1)}\big)^{-1} \lf(\bm{\wtb}-\bm{\wtb}^{(k,*)}\ri),
\end{align*}
 where $ \bm{\wtb} = ( \wtb^{(l)}_n: l=1,\hdots, L, n=1, \hdots, N_0)^\top$ and $\bm{\wtb}^{(k,*)}$ given by \eqref{eq:opt-wt}  . It follows that $\bm{\wtb}^{(k,\mathrm{min})} = \bm{\wtb}^{(k,*)}$.

\end{proof}	

\subsection{SMC Proposal}\label{sec:SMC}	
We now describe the SMC algorithm that gives {\em self-normalizing importance sampling estimator} of the form \eqref{eq:SMC-filter}.  The SMC algorithm generates particle-paths, $\tilde{\BX}^{(l,k-1)}_{[0,T]}, l=1,2,\hdots L$, sequentially between observation time points $t_j$, with resampling performed at these points when necessary to mitigate particle degeneracy by discarding particles with negligible weights and replicating those with larger weights, ensuring an accurate approximation of the target distribution $\filtdist_{b_{k-1}}(\cdot| \by_{t_1:t_{M_0}})$. For notational convenience, we will suppress the EM-iteration-index $k$, and describe the process of generating particle-paths for a given drift function $\drft$.

Recall that we assumed $\{t_1,t_2,\hdots, t_{M_0}\} \subset \{s_0,s_1,\hdots, s_{N_0}\}$ with $t_{m} =s _{n_m}$ for $m=1,2,\hdots,M_0.$ A particle-path $\tilde{\BX}_{[0,T]}$ is chosen from the proposal density $\propden$ on  $\R^{d\times (N_0+1)}$ of the form
\begin{align*}
 \propden_T(\bx_{0:N_0}) = \prod_{m=0}^{M_0-1}  \propden_{(t_m,t_{m+1}]}\lf(\bx_{n_m+1:n_{m+1}} | \bx_{0:n_m}\ri), \quad \bx_{0:N_0} = (\bx_{0:n_1}, \bx_{n_1+1:n_2}, \hdots, \bx_{n_{M_0-1}+1:n_{M_0}}). %\in \R^{d\times (N_0+1)}
\end{align*}	
In other words, having chosen the path $\tilde X_{[0,t_{m}]}$ until observation time $t_{m}$, the segment of the path $\tilde X_{(t_{m}:t_{m+1}]} = (\tilde{X}(s_{i}): n_m<i\leq n_{m+1})$ is chosen from a proposal density $\propden_{(t_m,t_{m+1]}}$ on $\R^{d\times (n_{m+1}-n_m)}$ (note that $n_{m+1}-n_m$ is the number of $s_n$ that falls between the observation time points $t_m$ and $t_{m+1}$).

 The proposal density of course needs to depend on the observation values $\by_{1:M_0}$. Indeed the optimal proposal distribution for generating the trajectory segment between $t_m$ and $t_{m+1}$ is simply the probability measure $\PP\lf(\BX_{(t_m,t_{m+1}]} \in 
 \cdot|X(t_m)=X(s_{n_m})=x_{i_j}, Y(t_{m+1}) =y_{m+1}\ri)$; that is, the optimal proposal density $\propden^{\mathrm{opt}}_{(t_m,t_{m+1}]}$ for the segment $(t_m,t_{m+1}]$  is
 
 \begin{align*}
 \propden^{\mathrm{opt}}_{(t_m,t_{m+1}]}\lf(\bx_{n_m+1:n_{m+1}} | \bx_{0:n_m}\ri) \equiv &\  \propden^{\mathrm{opt}}_{(t_m,t_{m+1}]}\lf(\bx_{n_m+1:n_{m+1}} | x_{n_m},  y_{m+1}\ri)\\
  =&\ \f{\obsden(y_{m+1}|x_{n_{m+1}})\prod_{i>n_m}^{n_{m+1}}\Xden(x_i|x_{i-1})}{p(y_{m+1}|x_{n_m})},
\end{align*}	 
where the normalizing constant
\begin{align*}
p(y_{m+1}|x_{n_m})= \int_{\R^{d\times (n_{m+1}-n_m)}} \obsden(y_{m+1}|x_{n_{m+1}})\prod_{i>n_m}^{n_{m+1}}\Xden(x_i|x_{i-1}) d\bx_{n_m+1:n_{m+1}}
\end{align*}	
is  the conditional density function of $Y(t_{m+1})$ given $X(t_m) = X(s_{n_m})=x_{n_m}$ evaluated at the point $y_{m+1}$. 

Sampling from this density is almost always infeasible, except when the drift function $b$ has a particularly simple form. This difficulty is even more pronounced in our case of SMC within EM algorithm, where the drift function for the $k$-th EM iteration is estimated as a kernel expansion of the form \eqref{eq:b-knl-exp}, based on the results from the previous $(k-1)$-th iteration. However it shows the importance of drawing each segment of the sample path $\tilde X_{(t_m,t_{m+1}]}$ from a proposal importance distribution depending not just on `starting value' at the initial point at $t_m$, but also on the observation / data $y_{m+1}$ at time $t_{m+1}$. Thus we choose of the form:
 \begin{align}\label{eq:prop-0}
 \propden_{(t_m,t_{m+1]}}\lf(\bx_{n_m+1:n_{m+1}} | \bx_{0:n_m}\ri) \equiv  \propden_{(t_m,t_{m+1]}}\lf(\bx_{n_m+1:n_{m+1}} | x_{n_m},  y_{m+1}\ri) = \prod_{i>n_m}^{n_{m+1}} \seqpropden_{m,i}(x_{i}|x_{i-1},y_{m+1})
 \end{align}	
which enables sampling $\tilde X_{(t_m,t_{m+1}]}$ according to `Markovian dynamics' starting from the `initial value' at $t_m$ and depending on the end observation $y_{m+1}$.  Importantly, due to the sequential structure of the proposal, the (unnormalized) weights for the $l$-th particle-path, $\tilde X^{(l)}$, at observation times $\{t_m\}$ can be recursively calculated as 
 \begin{align}\label{eq:wts-1}
 	\tilde w^{(l)}_{t_{m+1}} = \tilde w^{(l)}_{t_{m}} \f{\obsden\big(y_{m+1}|\tilde X^{(l)}(t_{m+1})\big) \prod\limits_{i>n_m}^{n_{m+1}} \Xden\big(\tilde X^{(l)}(s_{i})|\tilde X^{(l)}(s_{i-1})\big)}{\prod_{i>n_m}^{n_{m+1}} \seqpropden_{m,i}\big(\tilde X^{(l)}(s_{i})|\tilde X^{(l)}(s_{i-1}), y_{m+1}\big)}.
 \end{align}	 

\subsection*{ Choice of the proposal $\seqpropden_{m,i}$:} 
We consider the case when the observation model is given by \eqref{eq:obs}, that is, $\obsden(\cdot| x_{n_m}) = \No_{d_0} (\cdot|\omat x_{n_m}, \Sigma_{noise})$. Fix an $n_m<i\leq n_{m+1}.$ The proposal density $\seqpropden_{m,i}$ for generating $\tilde X(s_i)$ given $Y(s_{n_{m+1}})=Y(t_{m+1})$ will depend on the time points $s_i$ and $s_{n_{m+1}}$ through the difference $s_{n_{m+1}}-s_i = (n_{m+1}-i)\Delta.$ For a suitable choice of the density $\seqpropden_{m,i}$, we write
\begin{align*}
\seqpropden_{m,i}(x_{i}|x_{i-1}, y_{m+1})  \propto& \ \tilde \seqpropden^0_{m,i}(y_{m+1}|x_i, x_{i-1})\tilde \seqpropden^1_{m,i}(x_{i}|x_{i-1})\\
 \propto&\ \lf(\int_{\R^d}\tilde \seqpropden^{01}_{m,i}(y_{m+1}|x_{n_{m+1}}) \tilde \seqpropden^{02}_{m,i}(x_{n_{m+1}}|x_{i}, x_{i-1}) dx_{n_{m+1}}\ri) \tilde \seqpropden^1_{m,i}(x_{i}|x_{i-1}).
 \end{align*}
 Observe that the {\em natural choices} of the densities $\tilde \seqpropden^1_{m,i}$ and $\tilde \seqpropden^{01}_{m,i}$ are of course
\begin{align}\label{eq:smc-prop-1}
\begin{aligned}
\tilde \seqpropden^1_{m,i}(x_i|x_{i-1})  =&\ \Xden(x_{i}|x_{i-1}) = \No_d\big(x_{i}|x_{i-1}+ \drft(x_{i-1})\Delta, a(x_{i-1})\Delta\big),\\
 \tilde \seqpropden^{01}_{m,i}(y_{m+1}|x_{n_{m+1}}) =&\ \rho_{obs}(y_{m+1}|x_{n_{m+1}}) = \No_{d_0} (\cdot|\omat x_{n_{m+1}}, \Sigma_{noise})),
\end{aligned}
\end{align}
where recall $a=\dffun\dffun^\top$. 
Thus the  choice of the  proposal $\seqpropden_{m,i}(x_{i}|x_{i-1}, y_{m+1})$ for $n_m<i\leq n_{m+1}$ is simply determined by our choice of the density, $\tilde \seqpropden^{02}_{m,i}(x_{n_{m+1}}|x_{i}, x_{i-1})$. For our SDE model, computational efficiency motivates the use of a Gaussian distribution as the proposal density $\seqpropden_{m,i}$. The following result, which follows directly from the properties of conditional distributions of the multivariate normal (see \Cref{lem:normal}), plays a key role in its specific form.

\begin{lemma}
\label{lem:prop-gaussian}
Assume that $\tilde\seqpropden^1_{m,i}$, $\tilde\seqpropden^{01}_{m,i}$ are given by \eqref{eq:smc-prop-1}, and
\begin{align}
\label{eq:q02-choice}
\tilde \seqpropden^{02}_{m,i}(x_{n_{m+1}}|x_{i}, x_{i-1}) = \No_d\lf(x_{n_{m+1}}\big| x_{i}+\propmean_{m,i}^{02}(x_{i-1}), \propvar_{m,i}^{02}(x_{i-1})\ri),
\end{align}
where $\propmean_{m,i}^{02}(x_{i-1})$ and $\propvar_{m,i}^{02}(x_{i-1})$  depend on $x_{i-1}$, but do not depend on $x_i$.
Then the  proposal density $q_{m,i}$ is given by
\begin{align}
\label{eq:seq-prop-choice}
\seqpropden_{m,i}(x_{i}|x_{i-1}, y_{m+1}) = \No_{d}\big(x_i | x_{i-1} +\propmean_{m,i}(x_{i-1}, y_{m+1})\Delta, \propvar_{m,i}(x_{i-1})\Delta\big),
\end{align}
where
\begin{align} \label{eq:q-prop-para}
\begin{aligned}
\propvar_{m,i}(x_{i-1}) =&\  a(x_{i-1})\lf[I -\omat^\top\lf(\Sigma_{noise}+\omat \propvar_{m,i}^{02}\omat^\top+\Delta\  \omat a(x_{i-1})\omat^\top\ri)^{-1}\omat  \Delta a(x_{i-1}) \ri]\\
\propmean_{m,i}(x_{i-1}, y_{m+1})=&\ \drft(x_{i-1}) +  a(x_{i-1})\omat^\top\lf(\Sigma_{noise}+\omat \propvar_{m,i}^{02}\omat^\top+\Delta\  \omat a(x_{i-1})\omat^\top\ri)^{-1}\\
&\ \times \lf(y_{m+1}-\omat\lf(x_{i-1}+\propmean_{m,i}^{02}+\drft(x_{i-1})\Delta\ri)\ri).
%=&\ F_j(x_{s_j-1})\tilde \mu_j+\mfk{h}_j(x_{s_{j-1}},y_T)
\end{aligned}
\end{align} 
\end{lemma} 
\np
A  zeroth-order approximation of $\drft$ around $X(s_i)$ leads to the crude Euler-approximation of the form 
\begin{align}\label{eq:SDE-approx-0}
\begin{aligned}
X(s_{n_{m+1}}) =&\ X(s_i)+\int_{s_i}^{s_{n_{m+1}}}\drft(X(r))\ dr+ \int_{s_i}^{s_{n_{m+1}}}\dffun(X(r))\ dW(r)\\
\approx&\  X(s_i)+\drft(X(s_i))(s_{n_{m+1}}-s_i)+ \dffun(X(s_i))\ (W(s_{n_{m+1}})-W(s_i))
\end{aligned}
\end{align}
and a further approximation in the form of $\drft(X(s_i)) \approx \drft(X(s_{i-1}))$ and $\dffun(X(s_i)) \approx \dffun(X(s_{i-1}))$ to satisfy the requirement of \Cref{lem:prop-gaussian} leads to  $\tilde \seqpropden^{02}_{m,i}(x_{n_{m+1}}|x_{i}, x_{i-1})$ given by \eqref{eq:q02-choice} with 
\begin{align*}
\propmean_{m,i}^{02} = \drft(x_{i-1})(s_{n_{m+1}}-s_i), \quad \propvar_{m,i}^{02} = a(x_{i-1})(s_{n_{m+1}}-s_i)= a(x_{i-1})(n_{m+1}-i)\Delta.
\end{align*}
In the case of noise-free observations $\{y_m\equiv x_{n_m}\}$, this results in the {\em modified diffusion bridge} proposal of Durham and Gallant where for $n_m<i\leq n_{m+1}$ the proposal density density $q_{m,i}$ is given by \eqref{eq:seq-prop-choice} with
\begin{align}
\label{eq:DG-prop}
\propmean_{m,i}(x_{i-1}, y_{m+1} \equiv x_{n_{m+1}}) = \f{x_{n_{m+1}}-x_{i-1}}{s_{n_{m+1}}-s_{i-1}}, \quad  \propvar_{m,i}(x_{i-1}) = \f{s_{n_{m+1}}-s_i}{s_{n_{m+1}}-s_{i-1}}a(x_{i-1}).
\end{align}

\vs{.1cm}
\np
{\em Linear SDE approximation:} 
 As mentioned, proposals based on zeroth-order approximation of $\drft$  are not effective in our nonlinear setup. %In this paper we construct  $\tilde \seqpropden^{02}_{m,i}$ through a first-order Taylor expansion of $\drft(\cdot)$  producing a more accurate approximation in \eqref{eq:SDE-approx-0} that is particularly effective in the highly nonlinear framework of our paper. 
 To address this,  for $r \in (s_i, s_{n_{m+1}}]$, rather than approximating $b(X(r))$ by $b(X(s_i))$  as in \eqref{eq:SDE-approx-0}, we employ a first-order Taylor expansion of $b(X(r))$ around $X(s_i)$ to better capture the local behavior of the drift function. This gives 
\begin{align*}
	\drft(X(r)) \approx \drft(X(s_i)) + \deriv \drft(X(s_i)) (X(r)-(X(s_i)),
\end{align*}
where $Db(x) = ((\partial_j b_i(x): i,j=1,2,\hdots, d))$.
Making further approximations in the form of $\drft(X(s_i)) \approx \drft(X(s_{i-1}))$,  $\deriv \drft((X(s_i)) \approx \deriv \drft((X(s_{i-1}))$ and  $\dffun(X(s_i)) \approx \dffun(X(s_{i-1}))$  to satisfy the requirement of \Cref{lem:prop-gaussian} leads to an approximation of $X$ on the time interval $(s_i, s_{n_{m+1}}]$ with $X(s_{i-1})=x_{i-1}, \ X(s_i)=x_i$ by the process $x_i+\tilde X$, where $\tilde X$ satisfies the following linear SDE:
\begin{align*}
	\tilde X(t) =  \int_{s_i}^t \big(\drft(x_{i-1}) + \deriv \drft(x_{i-1})\tilde  X(r) \big)\ dr +  \s(x_{i-1}) (W(t)-W(s_i)), \quad \tilde X(s_i)=0.
\end{align*}
Since $\tilde X$ satisfies a linear SDE, its marginal distributions are  Gaussian given by, 
$$\tilde X(t) \sim \No_d\lf(\tilde \propmean(t-s_i, x_{i-1}), \tilde \propvar(t-s_i, x_{i-1})\ri),$$
where $\tilde \propmean(\cdot)\equiv \tilde \propmean(\cdot,x_{i-1}) $ and $ \tilde \propvar(\cdot)\equiv  \tilde \propvar(\cdot,x_{i-1})$ satisfy the ODE system:
\begin{align}\label{eq:prop-mean-var-ode}
	\begin{aligned}
	\f{d \tilde \propmean(t)}{dt} =&\ \drft(x_{i-1}) + \deriv \drft(x_{i-1})\tilde \propmean (t), \quad  \tilde \propmean(0) = 0\\
	\f{d  \tilde \propvar(t)}{dt}=&\ \deriv \drft(x_{i-1}) \tilde \propvar(t)+ \tilde \propvar(t)\deriv^\top\drft(x_{i-1})+a(x_{i-1}), \quad\tilde \propvar(0)=0
   \end{aligned}
\end{align}
The last equation is a differential Sylvester equation and can be solved by a variety of numerical methods. If $\deriv \drft(x)$ is invertible for each $x$, the solution (see \eqref{eq:lsde-mv}) can be expressed as
\begin{align}\label{eq:prop-mean-var-exp}
	\begin{aligned}
  \tilde \propmean(t, x_{i-1}) =&\ \lf(e^{\deriv \drft(x_{i-1})t}-I\ri))\lf(\deriv \drft(x_{i-1})\ri)^{-1}\drft(x_{i-1})\\
  \tilde \propvar(t, x_{i-1}) = &\ \int_0^t e^{\deriv \drft(x_{i-1})(t-r)}a(x_{i-1})e^{\deriv ^\top\drft (x_{i-1})(t-r)}\ dr.
  \end{aligned}
\end{align}
Using the vectorization operator, the integral for $\tilde \propvar$ can be computed in closed form to give 
\begin{align*}
 \ve\lf(\tilde \propvar(t, x_{i-1})\ri) = \lf(e^{\deriv \drft(x_{i-1}) \oplus \deriv \drft(x_{i-1})t}-I\ri) \lf(\deriv \drft(x_{i-1}) \oplus \deriv \drft(x_{i-1})\ri)^{-1}\ve(a(x_{i-1})),
\end{align*}
where recall $\oplus$ denotes the Kronecker sum.
Thus a first-order linear SDE-approximation of $X$ leads to $\tilde \seqpropden^{02}_{m,i}(x_{n_{m+1}}|x_{i}, x_{i-1})$ given by \eqref{eq:q02-choice} with 
\begin{align}\label{eq:prop-choice-lin-sde}
	\propmean_{m,i}^{02}(x_{i-1}) = \tilde \propmean(s_{n_{m+1}}-s_i, x_{i-1}), \quad \propvar_{m,i}^{02}(x_{i-1}) = \tilde \propvar(s_{n_{m+1}}-s_i, x_{i-1}).
\end{align}

\begin{remark} {\rm
Consider the case where  $\drft$ is given by a finite-sum kernel expansion of the form 
$ \drft(x) = \sum_{j=1}^J \knl(x,u_l) \beta_l, \ \beta_l \in \R^d.$
When the matrix-valued kernel, $\knl$, is of the form $ \knl(x,u) = \kappa_{0}(x,u)I_{d}$, where $\kappa_0$ is a real valued kernel (that is, $\drft(x) = \sum_{l=1}^L \kappa_0(x,u_l) \beta_l$),  the derivative, $\deriv\drft$, admits the  expression $ \deriv\drft(x) = \sum \limits_{l=1}^L \beta_L^\top \ot \nabla_x \kappa(x, u_l)$.
%\begin{align*}
% \deriv\drft(x) = \sum \limits_{l=1}^L \beta_L^\top \ot \nabla_x \kappa(x, u_l).
% \end{align*}
Notice that for a Gaussian kernel, $\kappa_0$, given by $\kappa_0(x,u) = \exp\lf(-\f{\|x-u\|^2}{c}\ri),\ x, u \in \R^d$, $\nabla_x \kappa_0(x,u) = -\f{2}{c} \kappa_0(x,u) (x-u).$
}
\end{remark}

\subsection{Inference Algorithms} \label{sec:algo}

We summarize our learning methodology in the following EM and SMC algorithms. In this paper, we perform resampling when the {\em effective sample size} (ESS) at an observation times $t_m$, defined by 
$$ESS_{m} =\f{1}{\sum_{l=1}^L (w^{(l)}_{t_{m}})^2} = \f{\lf(\sum_{l=1}^L \tilde w^{(l)}_{t_{m}}\ri)^2}{ \sum_{l=1}^L (\tilde w^{(l)}_{t_{m}})^2}, $$
 falls below a threshold $L_T$ (often chosen to be $L/2$), that is, if  $	ESS_m \leq L_T$; see \cite{Liu08, DoJo09}. Here $w^{(l)}_{t_m}$ and $\tilde w^{(l)}_{t_m}$ are normalized and unnormalized weights of the $l$-th particle-path at time $t_m$. Other resampling criteria can of course be incorporated.
%\begin{align}
%	\label{eq:ess-crit}
%	ESS_m \leq L/2.
%\end{align}

\vspace{0.2in}
\begin{algorithm}[H]
% \setstretch{1.35}
\DontPrintSemicolon
\KwIn{The data $\BY_{t_1:t_{M_0}} =(Y(t_1), Y(t_2), \hdots, Y(t_{M_0}))$, drift function $\drft$ and diffusion coefficient $\dffun$, 
a discretization time step \(\Delta\), number of particles $L$. }
\KwOut{ %A  sample-path of $X_{[0,t_m]}$ given the data $\BY_{t_1:t_m}$. 
$L$ sample paths $\{\BX^l_{[0,T]}: l =1,2,\hdots,L\}$ and their corresponding weights $\{w^l_{T}: l =1,2,\hdots,L\}.$
}
Initialize $X^l(0)$, where $l=1,2,\hdots,L$.\;
\While{$0\leq m<M_0$}{
\For{  \(l=1,2\hdots, L\)} {
      \While{$n_m< i \leq n_{m+1}$}{
     generate $X^l(s_{i})$  from the proposal transition density $\seqpropden_{m,i}(\cdot|X^l(s_{i-1}), Y(t_{m+1}))$ given by \eqref{eq:seq-prop-choice} with 
  $\mu_{m,i}^{02}(X^l(s_{i-1}))$ and $S^{02}_{m,i}(X^l(s_{i-1}))$ given in \eqref{eq:prop-choice-lin-sde}). \;
        Set $\BX^l_{[0,s_i]} = (\BX^l_{[0,s_{i-1}]}, X^l(s_i))$. \;
          Set $m=m+1$.\;
     }

     Compute the weight $w^l_{t_{m+1}} \propto \tilde w^l_{t_{m+1}}$ of each particle-path up to $t_{m+1}$  by \eqref{eq:wts-1}. \;
     }
     {\em Resample step:} If the weights of most particles are too low according to some degeneracy criterion (for example, if $ESS_{m+1} \leq L/2$),  resample $L$ particle-paths $X^l_{[0,t_{m+1}]} $ with replacement using the weights $\{w^l_{t_{m+1}}: l=1,2,\hdots,L\}$. Set $w^l_{t_{m+1}} =1/L$ after resampling.\;
   Set $ i= i+1$.\;
     
     }
     %Sample one complete path $X_{[0,t_m]}$ from \{$X^l_{[0,t_{m}]}: l =1,2,\hdots,L$\}.
     Get $L$ sample paths $\{X^l_{[0,t_{M_0}]}\equiv X^l_{[0,T]}: l =1,2,\hdots,L\}$ and their corresponding weights $\{w^l_{t_{M_0}}\equiv w^l_T: l =1,2,\hdots,L\}.$

\caption{ SMC Algorithm with proposal induced by linear SDE approximation }
\label{algo:SMC}
\end{algorithm}

\vspace{0.2in}
\begin{algorithm}[H]
% \setstretch{1.35}
\DontPrintSemicolon
\KwIn{The data $\BY_{t_1:t_m} =(Y(t_1), Y(t_2), \hdots, Y(t_{M_0}))$,
a discretization time step \(\Delta\), number of particle-paths $L$, number of EM-iterations $K$. }
\KwOut{ Drift function $b$. }
Initialize $X^l(0)$, where $l=1,2,\hdots,L$.\;
Initialize a drift function $b_0:\R^d \rt \R^d$.\;
\While{$1\leq k \leq K$}{
     Use SMC Algorithm (Algorithm \ref{algo:SMC}) with $\drft_{k-1}$ to get  $L$ sample paths $\{\tilde\BX^{(l,k-1)}_{[0:T]}: l =1,2,\hdots,L\}$ and their corresponding weights $\{w^{(l,k-1)}_{T}: l =1,2,\hdots,L\}$. \;
     Compute $\bm{\wtb}^{(k,*)}$ by \eqref{eq:opt-wt}.\;
     Compute $\drift_k$ by \eqref{eq:b-knl-exp}. \;
     Set $k=k+1$.
     }
  \caption{EM Algorithm for estimating $b$. }
\label{algo:EM-SMC}
\end{algorithm}

\np
{\bf A Bayesian approach:}
We also develop a hybrid Bayesian variant of the EM algorithm that allows the incorporation of different prior distributions on the coefficients $\bm{\wtb} =(\wtb^{(l,*)}_n:l=1,\hdots,L,n=0,1,\hdots,N_0)$ in the finite-kernel sum representation. Placing various priors on $\bm{\wtb}$ effectively broadens the class of regularization schemes applicable to the estimation problem. This connection between penalized optimization and Bayesian inference---well established in the literature \cite{Mack92}---relies on the observation that the negative log-posterior of $\bm{\wtb}$ under a given prior acts as a cost function. For instance, a Gaussian prior corresponds to an $L^2$-penalty in \eqref{eq:EM-1}, making the MAP estimate equivalent to a ridge-type regularized solution.

In Bayesian analysis, priors can serve multiple roles---incorporating external knowledge, regularizing ill-posed problems, and reducing complexity through sparsity or shrinkage, the latter being our primary focus here. A  survey of shrinkage and sparsity inducing priors can be found in \cite{VOM19}. In our setting, the drift function $\drft_k(\cdot)$ at $k$-th EM iteration is estimated via a kernel expansion centered at trajectory points $\BX^{l,k}_{[0,T]} \equiv \{X^{(l,k)}(s_n):n=1,2,\hdots, N_0, l=0,1,\hdots,L\}$ drawn (approximately) from the filtering distribution $\filtden_{\drft_{k-1}}(\cdot|\by_{1:M_0})$ using the previous iterate $\drft_{k-1}$. The coefficients  $\wtb^{(l,k)}_n$ determine the influence of each kernel center based on the  simulated high-frequency trajectory-points. While sparsity-inducing priors like spike-and-slab aim to discard irrelevant basis functions by driving many  $\wtb^{(l,k)}_n$ to exact zero (for a given iteration $k$), such aggressive selection may be suboptimal for drift estimation in SDEs, where smoothness and local adaptivity are crucial. In particular, hard sparsity may suppress informative local structure in regions of the state space that are visited infrequently, leading to underfitting or unstable estimates.

Shrinkage priors, such as the Student-t or Horseshoe, offer a more flexible alternative. They allow all simulated trajectory points to contribute to the estimate while adaptively shrinking negligible coefficients toward zero. This yields smoother, more robust estimates of $\drft(\cdot)$ in SDE learning. Importantly, shrinkage controls model complexity without the instability introduced by hard thresholding, allowing the retention of weak signals that may still play a critical role in the system’s dynamics.

From a computational and modeling standpoint, shrinkage priors also offer important advantages over spike-and-slab-type sparsity. In our EM framework, each iteration involves drawing $L$ full trajectories from the filtering distribution corresponding to the current drift estimate, and fitting the updated drift $\drft_k$ via kernel expansions over these points. A sparsity prior like spike-and-slab would potentially select a different subset of active coefficients $\wtb^{(l,k)}_n$ across different iterations and simulated trajectories, leading to inconsistent function representations across iterations. This lack of continuity in the estimated drift is not only statistically unstable but also physically unrealistic for dynamical systems governed by SDEs, where the drift typically vary smoothly with the state.

Shrinkage priors, by contrast, retain contributions from all basis elements while adaptively reducing the influence of less informative ones. This avoids erratic shifts in the estimated drift across iterations and better respects the underlying structure of the continuous-time dynamics. Moreover, shrinkage models bypass the combinatorial burden of selecting inclusion indicators, enabling faster and more stable EM updates. %Additionally, the heavy-tailed nature of such priors facilitates posterior contraction in high-dimensional settings, leading to more reliable function estimation under uncertainty.

In this paper we develop a Bayesian-EM algorithm based on $t$-prior on $\bm{\wtb}$, which has decent shrinkage properties. More advanced global-local shrinkage priors like Horseshoe and its variants can easily be implemented by following the recipe of \cite{GaMiZh23}. We use the following setup: 
\begin{align*}
\wtb^{(k,l,*)}_n|\l^l_n \sim \No(\cdot| 0,\lambda^l_n I_d), \quad \lambda^l_n \sim \SC{IG}(\hpa,\hpb),
\end{align*}
where $\SC{IG}(\hpa,\hpb)$ denotes the inverse-gamma distribution with shape and scale parameters $\hpa, \hpb > 0$.
For each iteration $k$, the parameter $\lambda^l_n$ controls the degree of shrinkage applied to the coefficient $\wtb^{(l,k)}_n$, thereby determining the influence of the trajectory point $X^{(l,k)}(s_n)$ sampled in the SMC algorithm. With the above prior, the conditional distribution of each parameter given the rest admits closed-form expressions, enabling a straightforward Gibbs sampler. However, the high dimensionality of the parameter space makes full posterior exploration at every EM iteration unnecessarily burdensome and inefficient---especially when the primary goal is to obtain a shrunk estimate rather than full posterior uncertainty quantification. To address this, at each EM iteration we compute the posterior mean of $\bm{\wtb}$ and use it as the updated estimate.
The resulting Bayeasin-EM algorithm is  summarized below. 

\vspace{0.2in}
\begin{algorithm}[H]
	\DontPrintSemicolon
	% \setstretch{1.35}
	\KwIn{The data $\BY_{t_1:t_m} =(Y(t_1), Y(t_2), \hdots, Y(t_m))$,  
		a discretization time step \(\Delta\), number of particle-paths $L$. }
	\KwOut{ Drift function $b$. }
	Initialize $X^l(0)$, and draw $\l^{l,0}_{n}$ from $\SC{IG}(\hpa,\hpb)$ for $l=1,2\hdots,L$ and $n=0,1,\hdots,N_0$. \;
	Initialize a drift function $b_0:\R^d \rt \R^d$.\;
	\While{$1\leq k \leq K$}{
		Use SMC Algorithm to get $\nu^{SMC}_{\drift_{k-1}}(\cdot| \BY_{t_1:t_m}) = \sum_{l=1}^{L} w^{(l,k-1)}_{T} \delta_{\tilde\BX^{(l,k-1)}_{[0:T]}}$. \;
		Compute $	\bm{\wtb}^{(k,*)} = \bcm^{(k-1)}\bm{\SC{K}_0}^{(k-1)} \sdg^{(k-1)} \bm{\vartheta}^{(k-1)}, $ where $	 \big(\bcm^{(k-1)}\big)^{-1} = \Delta  \big(\bm{\SC{K}_0}^{(k-1)}\big)^\top \sdg^{(k-1)} \bm{\SC{K}_0}^{(k-1)}+\text{diag}(1/\l^{(l,k-1)})\ot I_d$. \;
%		\begin{align*}
%	 \big(\bcm^{(k-1)}\big)^{-1} = \Delta  \big(\bm{\SC{K}_0}^{(k-1)}\big)^\top \sdg^{(k-1)} \bm{\SC{K}_0}^{(k-1)}+\text{diag}(1/\l^l_n)\ot I_d\\
%		 \bcm^{-1} = \Delta\bm{\SC{K}}_0^\top\sdg\bm{\SC{K}}_0 + \eta^{-1}
%		\end{align*}
		Compute $\drift_k$ by \eqref{eq:b-knl-exp}.\; % $\drift_k(u) = \sum_{l=1}^L\sum_{n=0}^{N_0} \knl(u, X^{l}(s_n)) \wtb^{(k,l,*)}_n$. \;
		Generate $\l^{(l,k)}_n \sim \SC{IG}(\hpa + \frac{d}{2}, \hpb + \frac{1}{2} (\wtb^{(l,k,*)}_n)^\top \knl(X^{(l,k-1)}(s_n), X^{(l,k-1)}(s_{n})) \wtb^{(l,k,*)}_n)$.\;
		$k=k+1$.
	}
	\caption{A Bayesian-EM Algorithm for estimating $b$. }
	\label{algo:Gibbs-EM-SMC}
\end{algorithm}

\subsection{Simulation studies} \label{sec:simu}
We demonstrate the effectiveness of our algorithm on three one-dimensional and two multi-dimensional  SDE models. For each model, synthetic data is generated as follows: a discrete trajectory is simulated over the interval $[0, 40]$ with step size $\Delta = 0.025$; additive noise is then applied to this path, and a subsample---e.g., $1/3$, $1/5$, or $1/10$ of the total points---is selected to serve as the observed noisy and sparse data. Using these observations, we apply our algorithm to estimate the coefficient vector $\bm{\wtb}$, yielding an estimated drift function $\hat \drft$, which we visually compare to the true drift $\drft$.

In addition to visual inspection, the small mean squared error (MSE) between $\hat \drft$ and $\drft$ quantitatively supports the accuracy of our method. Further validation comes from comparing the stationary distributions of the true and estimated SDEs through {\em Kolmogorov metric}, $\sup_{x}|F_{st}(x) - \hat F_{st}(x)|$, where $F_{st}$ and $\hat F_{st}$ are the CDFs of the true and estimated SDEs (driven by $\hat \drft$), respectively. The close agreement between these distributions demonstrates the predictive capability of the learned model even beyond the range of observed data, and indicates that the alignment of $\hat \drft$ with $\drft$ stems from genuine learning via accurate algorithm rather than overfitting.

We used Gaussian kernels for our simulation studies. Specifically, for the 1-D models, we used the kernel $\kappa_0(x,y) = 10\exp(-(x-y)^2/2)$ and for the multidimensional Michaelis-Menten kinetics in Model 4, we used $\knl_0 = \kappa_0 I_{3}$, for SIR in Model 5, we used  $\knl_0 = \kappa_0 I_{2}$. And for the SMC algorithm, we choose the number of particles to be 6 and use 3 of them with highest weight to simplify the calculation. 

\subsubsection{One-dimensional SDE models}
We consider three one-dimensional SDE models
\begin{enumerate}
\item {\bf Model 1 - Double-well potential SDE:} This is an overdamped Langevin SDE describing the motion of a particle in a double-well potential given by $u(x) = x^4 - 2x^2$. The particle's dynamics are influenced by two components: a  drift term $b(x) = -u'(x) = 4(x - x^3)$ and random fluctuations modeled by additive Brownian noise leading to the SDE 
$$dX(t) = 4X(t)(1 - X^2(t)),dt + \vas dW(t).$$
 The potential $u(x)$ has two wells (i.e., local minima) located at $\pm 1$, and the stochastic noise occasionally drives transitions between these wells. As a result, the dynamics are strongly non-linear. SDEs of this type also arise in mathematical finance, molecular dynamics, and climate science, where systems exhibit metastable behavior and rare transitions between stable states.  The double-well structure induces a bimodal stationary distribution with density: $\pi_{st}(x) \propto \exp\left(\frac{2x^2 - x^4}{2\vas^2}\right)$.

\item {\bf Model 2 - Variant of Double-well potential SDE:} This is a variant of the double-well potential SDE, incorporating multiplicative noise. The SDE for this model is given by:
\begin{align*}
	dX(t) = X(t)(1-X^2(t))dt + \vas \sqrt{1+X(t)^2} dW(t).
\end{align*}
The multiplicative noise introduces additional complexity to the nonlinear dynamics of the original double-well process. The stationary density of this SDE is given by 
\begin{align}
	\pi_{st}(x) \propto \vas^{-2}(1+x^2)^{2\vas^{-2}-1}\exp\lf(-x^2/\vas^2\ri).
\end{align}
The stationary distribution is bimodal when $\vas < 1$ and becomes unimodal if $\vas \geq 1$, with sharper peaks as $\vas$ increases. %In our experiments, we set $\vas = 1$.

\item {\bf Model 3 - Gamma SDE:} Our last one-dimensional SDE model is given by 
\begin{align*}
	dX(t) = \lf(\frac{9}{x} - 5\ri)dt + \vas dW(t).
\end{align*}
Its stationary density is given by $\pi_{st}(x) \propto \vas^{-2} x^{18}\exp\lf(-10x/\vas^2\ri)$
%\begin{align}\label{eq:gamma}
%	\pi_{st}(x) \propto \vas^{-2} x^{18}\exp\lf(-10x/\vas^2\ri).
%\end{align}
which corresponds to the density function of a Gamma distribution that has a shape parameter of 19 and a rate parameter of $10/\vas^2$, leading to a unimodal distribution that exhibits a peak at a $1.8\vas^2$. %We take $\vas = 1$ for our experiment.

\end{enumerate}

\np
For each of the these SDEs, our observation model is given by 
\begin{align*}
    \obsY(t) = X(t) +\vep_i, \quad \text{ with } \vep_i \stackrel{iid}\sim \No(0, \sigma^2_{noise}),
\end{align*}
We take $ \sigma_{noise} = 0.01$. We also take the diffusion parameter, $\vas = 1$ for all our experiments.

Figures \ref{fig:sparse-dw}, \ref{fig:sparse-dw-v-1} and \ref{fig:sparse-gamma} demonstrate the accurate performance of our algorithms. In each picture, panels {\bf a}, {\bf b}, {\bf c} compare $\drft$ with $\hat \drft$ with data comprising of a noisy sample of 1/3, 1/5, 1/10 of the total trajectory-points, respectively, respectively; panels {\bf d}, {\bf e}, {\bf f} compare the  stationary densities the SDEs driven by  $\drft$ and $\hat \drft$ under the same sampling fractions; and panels {\bf g}, {\bf h}, {\bf i}  compare the corresponding CDFs for the three SDE models. The MSE and Kolmogorov metrics obtained using Algorithm \ref{algo:Gibbs-EM-SMC} can be found in Table \ref{table:obs}. We also compared the Bayesian-EM algorithm (Algorithm \ref{algo:Gibbs-EM-SMC}) with the standard EM algorithm (Algorithm \ref{algo:EM-SMC}) and the results are reported in Table \ref{table:modified-em-method}.  While the overall performance is comparable, the Bayesian version often has an edge over the standard EM. Importantly, as shown in Figure \ref{fig:modified-em-hist}, the Bayesian EM algorithm with a t-prior achieves significantly better shrinkage. As expected, the performance improves as the sparsity of data decreases, that is, as more samples are observed. 

\begin{figure}[H]
	\centering
	\includegraphics[width=0.78\textwidth]{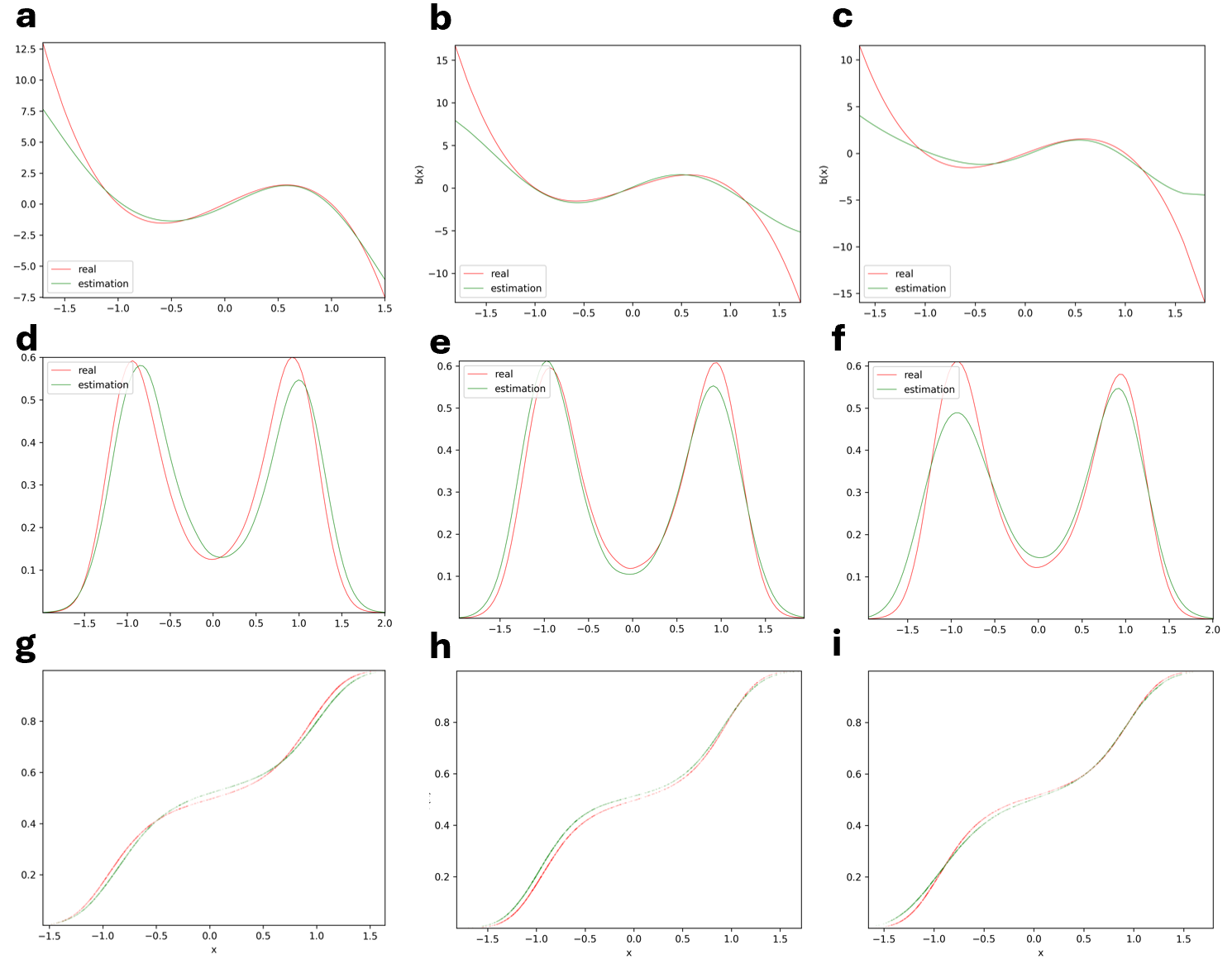}
	\caption{Model 1 - Double-well potential SDE. }
	\label{fig:sparse-dw}
\end{figure}

\begin{figure}[H]
	\centering
	\includegraphics[width=0.78\textwidth]{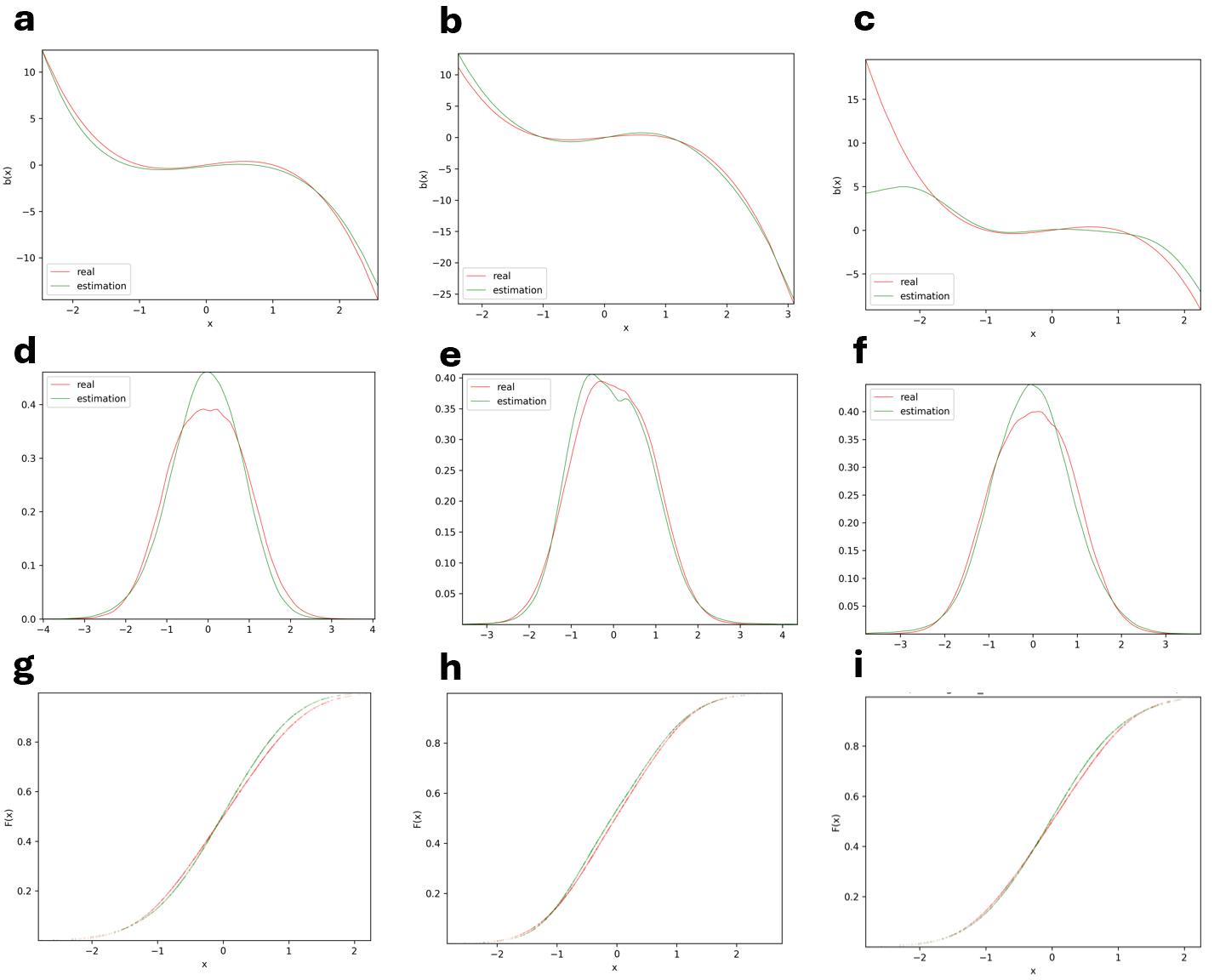}
	\caption{Model 2 - Variant of Double-well potential SDE.
	}
	\label{fig:sparse-dw-v-1}
\end{figure}

\begin{figure}[h!]
	\centering
	\includegraphics[width=0.78\textwidth]{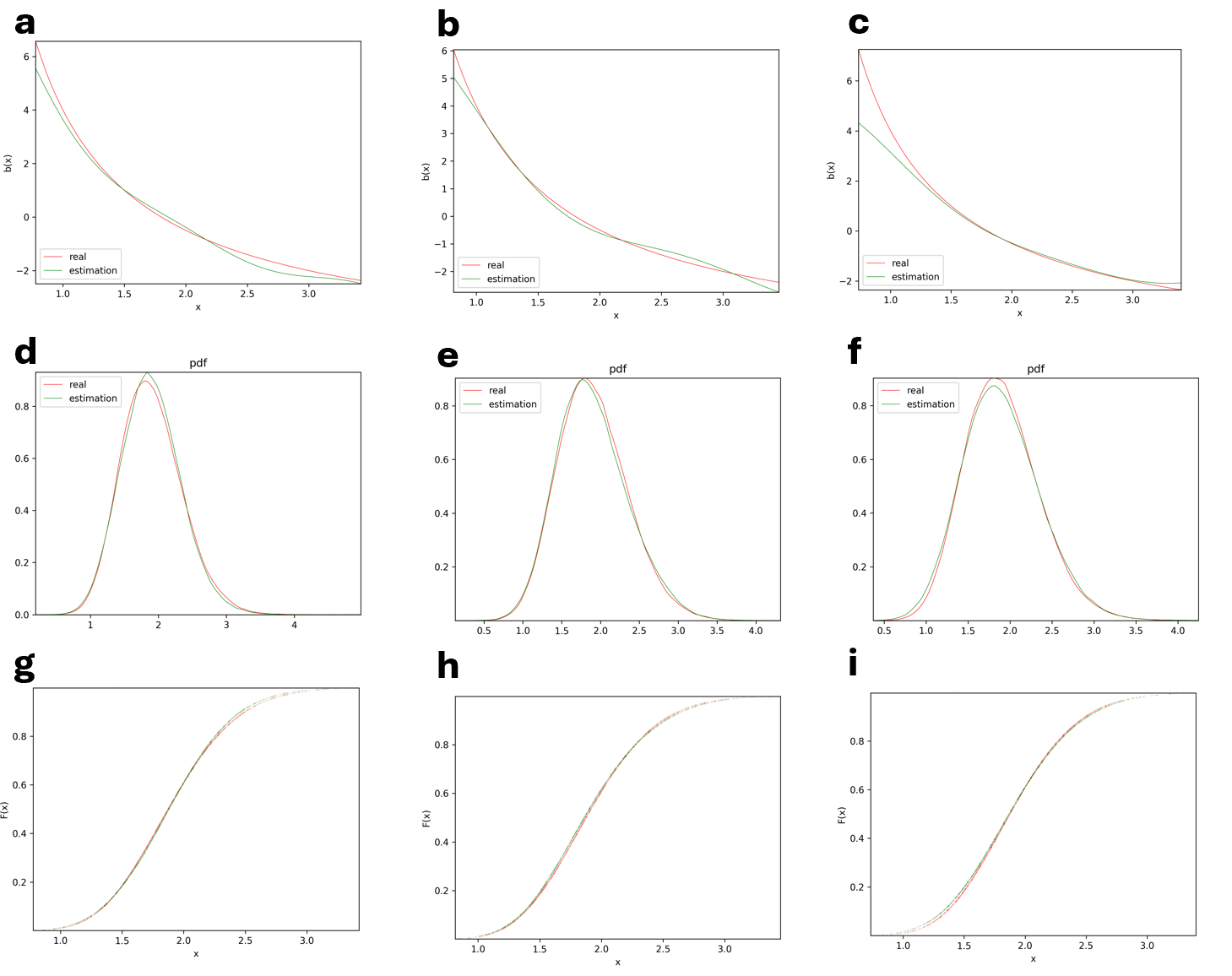}
	\caption{Model 3 - Gamma SDE.}
	\label{fig:sparse-gamma}
\end{figure}

\begin{table}[H]
\centering
\caption{MSE of $\hat \drft$ and Kolmogorov metric with different amounts of observed data}
\setlength\extrarowheight{6pt}
\begin{tabular}{|c|c|c|c|c|c|}
\hline
Model & Metric & 1/3 & 1/5 & 1/10 & 1/20 \\ \hline \hline
\multirow{2}{*}{Model 1} & MSE               & 0.286 & 0.488 & 0.86 & 1.424 \\
                         & Kolmogorov & 0.143 & 0.161 & 0.158 & 0.189 \\ \hline
\multirow{2}{*}{Model 2} & MSE               & 0.455 & 0.498 & 0.55 & 0.662 \\
                         & Kolmogorov  & 0.092 & 0.086 & 0.086 & 0.083 \\ \hline
\multirow{2}{*}{Model 3} & MSE               & 0.094 & 0.136 & 0.139 & 0.252 \\
                         & Kolmogorov  & 0.053 & 0.046 & 0.055 & 0.063 \\ \hline
\end{tabular}
\label{table:obs}
\end{table}

\begin{table}[H]
\centering
\caption{MSE of $\hat \drft $ and Kolmogorov metric using EM and Bayesian-EM algorithms}
\begin{tabular}{|c|c|cc|cc|}
\hline
\multirow{2}{*}{Model} & \multirow{2}{*}{Metric} & \multicolumn{2}{c|}{1/5 Observations} & \multicolumn{2}{c|}{1/10 Observations} \\
\cline{3-6}
                       &                         & EM & Bayesian-EM & EM & Bayesian-EM \\
\hline \hline
\multirow{2}{*}{Model 1} 
  & MSE               & 0.478 & 0.488    & 0.968    & 0.86 \\
  & Kolmogorov        & 0.169          & 0.161 & 0.188    & 0.158 \\
\hline
\multirow{2}{*}{Model 2} 
  & MSE               & 0.646          & 0.498 & 0.743   & 0.55 \\
  & Kolmogorov        & 0.083 & 0.086    & 0.072 & 0.086 \\
\hline
\multirow{2}{*}{Model 3} 
  & MSE               & 0.106 & 0.136   & 0.095 & 0.139 \\
  & Kolmogorov        & 0.05          & 0.046 & 0.046 & 0.055 \\
\hline
\end{tabular}
\label{table:modified-em-method}
\end{table}

\begin{figure}[h!]
	\centering
	\includegraphics[width=0.7\textwidth]{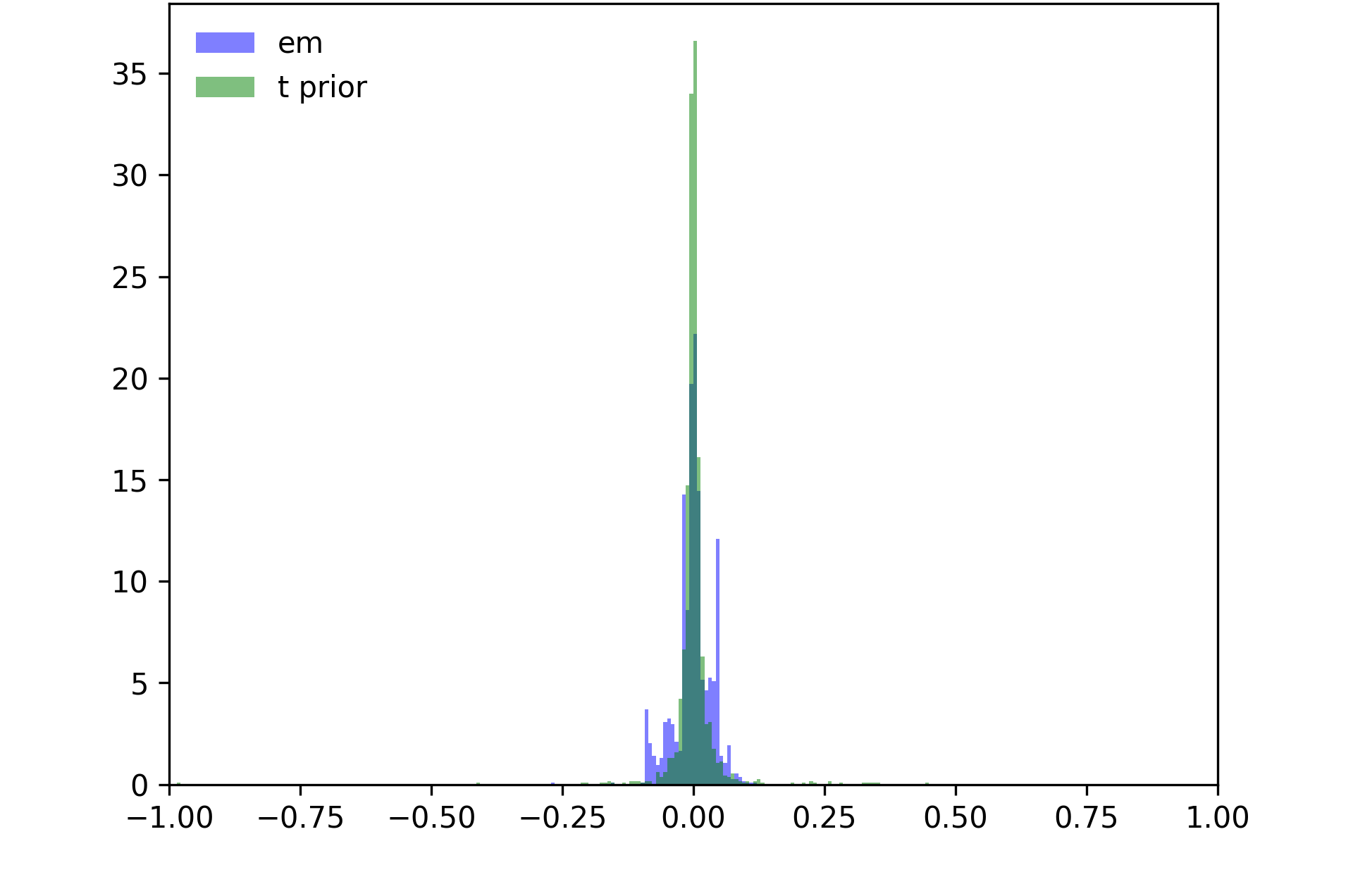}
	\caption{Histogram of $\beta_i$’s with EM and modified EM algorithms.}
	\label{fig:modified-em-hist}
\end{figure}

\subsubsection{Model 4: Michaelis-Menten Kinetics}
The Michaelis--Menten model is a foundational framework in enzymatic kinetics that describes the rate of enzymatic conversion of a substrate into a product via an intermediate enzyme--substrate complex (\cite{MiMe1913, Sri21}). This model captures essential nonlinearities in biochemical reactions and serves as a basis for more complex regulatory mechanisms in systems biology. The reaction system can be represented as
\begin{equation*}
	E + S \xrightleftharpoons[k_{2}]{k_{1}} ES, \quad ES \xrightleftharpoons[k_{m2}]{k_{m1}} E + P,
\end{equation*}
where \( E \) is the free enzyme, \( S \) is the substrate, \( ES \) is the enzyme--substrate complex, and \( P \) is the product. The system’s full state at time \( t \) is described by the vector
$
X(t) = \big(X_E(t),\, X_S(t),\, X_{ES}(t),\, X_P(t)\big)^\top.
$

Due to conservation of the total enzyme, the system satisfies
$
X_E(t) + X_{ES}(t) = X_E(0) + X_{ES}(0) \equiv C,
$
allowing for a reduced three-dimensional state space. We continue to denote the reduced state by
$
X(t) = \big(X_E(t),\, X_S(t),\, X_P(t)\big).
$
Accurate estimation of the parameters in this model is essential for understanding enzyme efficiency and designing effective inhibitors. Under idealized conditions, including the assumption that the system is well-mixed, the dynamics follow the law of mass action kinetics. This results in the drift function \( \drft(\cdot) \), given by (as a column vector)
\begin{align*}
	b(x) = (-k_1 x_E x_S - k_{m2} x_E x_P + (k_{m1} + k_2) x_{ES},\ -k_1 x_E x_S + k_{m1} x_{ES},\ k_2 x_{ES} - k_{m2} x_E x_P)^\top,
\end{align*}
with \( x = (x_E, x_S, x_P) \in \mathbb{R}^3 \), and \( x_{ES} = C - x_E \) due to conservation of total enzyme. We generate a discretized trajectory from the corresponding SDE with diffusion coefficient $\dffun(x)\equiv 0.1 I_3$ and conservation constant $C=2$ over time interval $[0,60]$ using the step size $\Delta=0.025$. Our observation model is given by 
$\obsY(t) = X(t) +\vep_i$ with   $\vep_i \stackrel{iid}\sim \No(0, \Sigma_{noise})$. We keep the measurement noise low, $\Sigma_{noise} = 10^{-10}I$, and use Algorithm \ref{algo:Gibbs-EM-SMC} to estimate the entire drift function $b$ from partially observed data comprising of 1/5 of the total (noisy) trajectory-points.

Figure \ref{fig:spares-mm} illustrates the performance of the estimation. Panels \textbf{a}, \textbf{b}, and \textbf{c} plot slices of the first, second, and third components of the estimated drift function $\hat{b}$ alongside the true drift function $b$, with the $z$-coordinate fixed at $1.073$. Panel \textbf{d} shows a two-dimensional slice of \textbf{a} at $y = 1.060$, providing further insight into the accuracy of the estimation.  The strong performance of our algorithm  is corroborated by the low MSE value of $0.003787$.

\begin{figure}[H]
	\centering
	\includegraphics[width=0.9\textwidth]{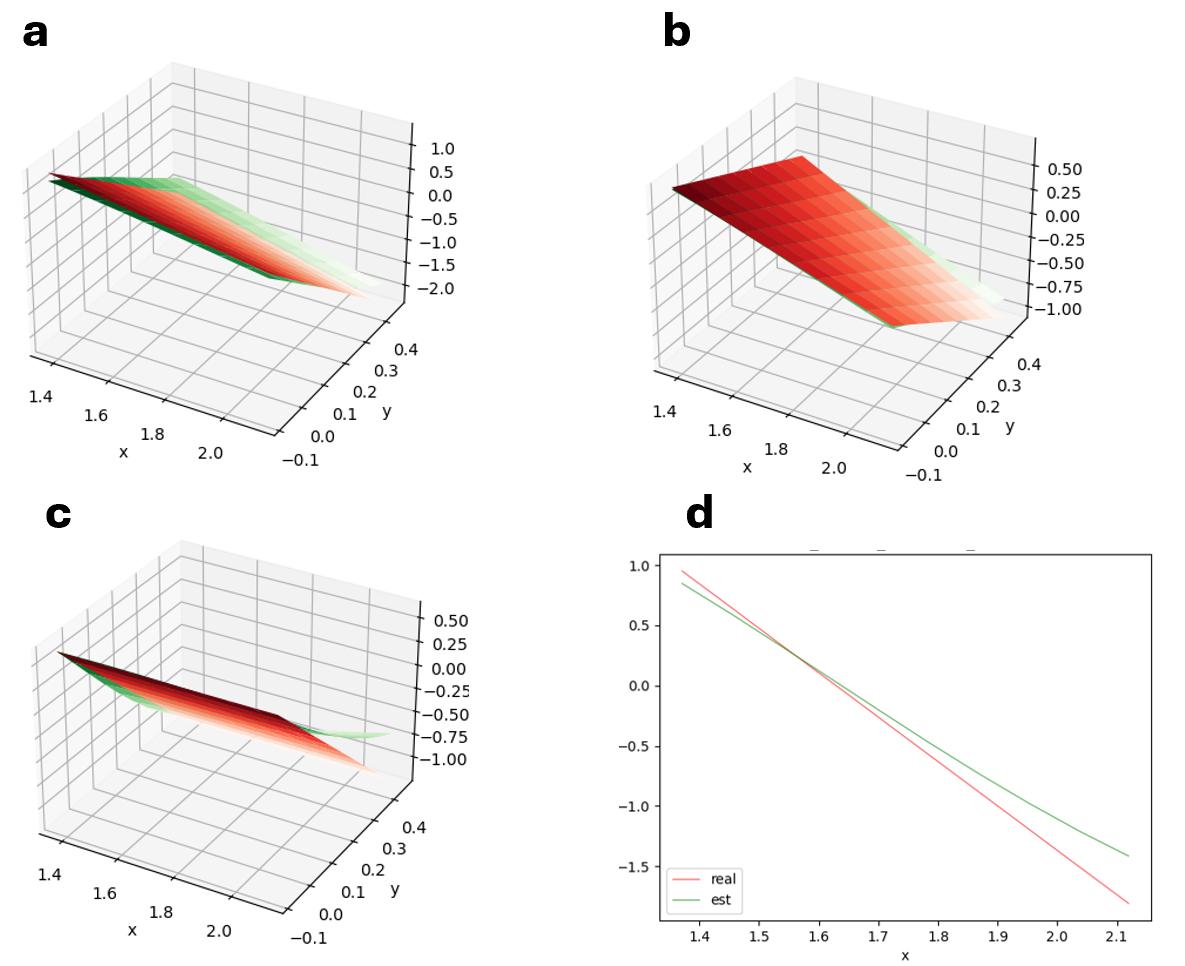}
	\caption{Estimation of Michaelis-Menten Kinetics model with t-prior.}
	\label{fig:spares-mm}
\end{figure}

\subsubsection{Model 5: SIR Model}
The SIR model is a classical and widely-used framework in epidemiology for modeling the spread of infectious diseases in a closed population. It captures the transition of individuals between three compartments: susceptible (S), infected (I), and recovered (R). Accurate estimation of the model, such as the transmission and recovery rates, is crucial for understanding epidemic dynamics, predicting disease spread, and informing public health interventions.

The system can be expressed via the following reaction network:
\begin{equation*}
    S + I \xrightarrow{\beta} 2I, \quad I \xrightarrow{\gamma} R,
\end{equation*}
where \( \beta \) denotes the transmission rate, and \( \gamma \) is the recovery rate. Letting \( X_S(T), X_R(t), X_I(t) \) denote the state of the system, that is proportion of individuals in each component, at time \( t \),  the model satisfies the conservation law \( X_S(t) + X_I(t) + X_R(t) = 1 \). This allows a reduced two-dimensional representation \( X(t) = (X_S(t), X_I(t)) \). The drift function $\drft$ of the system, when law of mass action holds, is given by $\drft(x)=(-\beta x_S x_I,  \beta x_S x_I - \gamma x_I)^\top$ with \( x = (x_S, x_I) \in \mathbb{R}^2 \). 

We generate a discretized trajectory from the corresponding SDE with $\beta=0.5, \gamma=0.6$ and diffusion coefficient $\dffun(x)\equiv 10^{-6} I_2$ over time interval $[0,60]$ using the step size $\Delta=0.025$. Our observation model is given by 
$\obsY(t) = X(t) +\vep_i$ with   $\vep_i \stackrel{iid}\sim \No(0, \Sigma_{noise})$. We keep the measurement noise low, $\Sigma_{noise} = 10^{-100}I$, and use Algorithm \ref{algo:Gibbs-EM-SMC} to estimate the entire drift function $b$ from partially observed data comprising of 1/5 of the total (noisy) trajectory-points.

Figure \ref{fig:sparse-sir} demonstrates the performance of the estimation. Panels \textbf{a} and \textbf{b} show the first and second components of the estimated drift function $\hat{b}$ compared with the true drift function $b$; panels \textbf{c} and \textbf{d} present heatmaps of the first and second components of the estimated and true drift functions, $\hat{b}$ and $b$, highlighting regions where the estimates closely match the true values. The accuracy of the algorithm in recovering the underlying dynamics is evident from the low MSE value of $5.577 \times 10^{-8}$.

\begin{figure}[h!]
	\centering
	\includegraphics[width=0.85\textwidth]{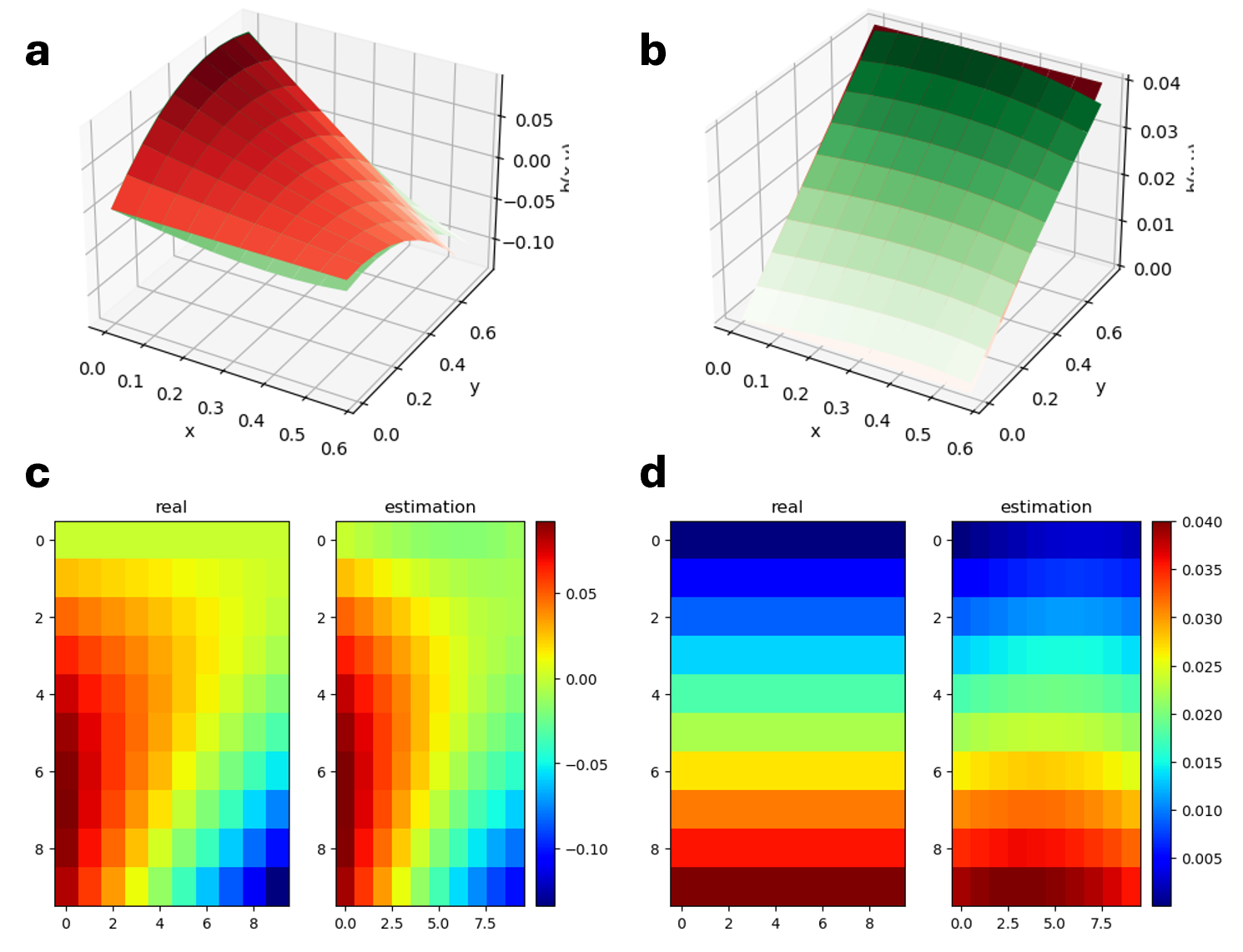}
	\caption{Estimation of SIR model with t-prior. }
	\label{fig:sparse-sir}
\end{figure}

\setcounter{section}{0}
\setcounter{theorem}{0}
\setcounter{equation}{0}
\renewcommand{\theequation}{\thesection.\arabic{equation}}

\appendix
\section{Appendix}
\label{sec:app}
\subsection{RKHS of vector-valued functions}
The RKHS of vector-valued functions is defined similarly to the scalar-valued case, with the main difference being that the associated kernel $\knl$ is matrix-valued (for details, see \cite{MiPo05, AlRoLa12}).

\begin{definition} \label{def:matRK}
	Let $\meU$ be an arbitrary space.
	A symmetric function $\knl: \meU\times \meU \rt \R^{n\times n}$ is called a {\em reproducing kernel} if for any $u, u' \in \meU$, $\knl(u,u')$ is a $n\times n$ p.s.d. matrix.
	
	The RKHS associated with $\knl$ is the Hilbert space $\Hsp_{\knl}$ consisting of functions $h: \meU \rt \R^n$ such that, for every $u \in \meU$ and every (column) vector $z \in \R^n$, the following two properties hold: (i) the mapping $u' \rt \knl(u',u)z$ is an element of $\Hsp_{\knl}$, and (ii) $\<h, \knl(\cdot, u)z\> = h(u)^\top z.$
\end{definition}
Property (ii) is the reproducing property of the kernel $\knl$ in the vector-valued setting. As in the scalar case, the Moore-Aronszajn theorem extends to this framework and shows that $\Hsp_{\knl} = \overline{\text{Span}}\{\knl(\cdot,u): u \in \meU\}$. The closure is taken with respect to the norm $\|\cdot\|_{\knl}$, which in this case is defined as follows: for $h = \sum_{j=1}^l \knl(\cdot, u_j)z_j, \ z_j \in \R^n$, 
$\|h\|^{2}_{\knl} \dfeq \sum_{i,j=1}^l z_i^\top \knl(u_i,u_j)z_j.$

Notice that for any $h \in \Hsp_{\knl}$ and $u,u' \in \R^d$, the following estimates hold:
\begin{align}\label{eq:rkhs-l2norm-est}
\begin{aligned}
\|h(u)\| \leq&\  \|h\|_{\Hsp_{\knl}}\|\knl(u,u)\|^{1/2}_{\mrm{op}}\\
\|h(u)-h(u')\| \leq &\ \|h\|_{\Hsp_{\knl}} \|\knl(u,u)-2\knl(u,u')+\knl(u',u')\|^{1/2}_{\mrm{op}}
\end{aligned}
\end{align}
where recall for a p.s.d matrix $A$, $\|A\|_{\mrm{op}} = \sup_{z \in \R^d:\|z\|=1} z^\top Az=$ max
 eigenvalue of $A$. Indeed, 
\begin{align*}
%\begin{aligned}
\|h(u)-h(u')\| =&\ \sup_{z \in \R^d:\|z\|=1} z^\top(h(u)-h(u')) = \sup_{z \in \R^d:\|z\|=1} \<h, (\knl(\cdot,u)-\knl(\cdot,u'))z\>\\
\leq &\ \|h\|_{\Hsp_{\knl}} \sup_{z \in \R^d:\|z\|=1}\|(\knl(\cdot,u)-\knl(\cdot,u'))z\|_{\Hsp_{\knl}} \\
 =&\ \|h\|_{\Hsp_{\knl}} \lf(\sup_{z \in \R^d:\|z\|=1} z^\top\big(\knl(u,u)-2\knl(u,u')+\knl(u',u')\big)z\ri)^{1/2} \\
=&\ \|h\|_{\Hsp_{\knl}} \|\knl(u,u)-2\knl(u,u')+\knl(u',u')\|^{1/2}_{\mrm{op}}
%\end{aligned}
\end{align*}
The first bound is easier.

\subsection{Auxiliary Results}

\begin{lemma}\label{lem:empty-int}
Let $\Hsp$ be an infinite-dimensional Hilbert space and  $\mathbb{A} \subset \Hsp$ a subset.  Suppose $\mathbb{A}$ has the property that  weak and strong convergences are equivalent on $\mathbb{A}$. Then $\mathbb{A}$ has empty interior in the strong topology of $\Hsp$.
\end{lemma}	

\begin{proof}
	Suppose that $\mathbb{A}^\circ \neq \emptyset$. Then there exists an $h_0 \in \mathbb{A}$ and $r_0>0$ such that $B(h_0,r_0) \subset \mathbb{A}.$ But every ball in an infinite-dimensional Hilbert space contains a  weakly convergent sequence which is not strongly convergent. For example, if $\{e_n\}$ is an orthonormal sequence of $\Hsp$, then $h_0+r_0e_n/2 \in B(h_0,r_0)$, and $h_0+r_0e_n/2 \stackrel{w}\Rt h_0$, as $n\rt \infty$. But clearly,   $\{h_0+r_0e_n/2\}$ does not converge in norm to $h_0$.	This violates the property of $A$ leading to a contradiction.

\end{proof}

\begin{lemma}\label{lem:mmt-bd-X}	
Let $\{\drft^{(L)}\} \subset \Hsp_{\knl} $ be a family such that $\sup_{L}\|\drft^{(L)}\|_{\Hsp_{\knl}} < \infty$. Suppose that \Cref{assum:sde-knl}-\ref{item-assum-knl} \& \ref{item-assum-dffun} hold. Then there exist a constant $\const_{p,1} \equiv \const_{p,1}(N_0)$ such that for all $0\leq n\leq N_0,$
$\sup_{L}\EE_{\drft^{(L)}}|X(s_{n})|^p  \leq \const_{p,1}.$
\end{lemma}
\begin{proof}
Notice that \eqref{eq:sde-disc-0} implies 
	\begin{align*}
		\|X(s_n)\|^p \leq& 3^{p-1}\lf( \|X(s_{n-1})\|^p + \|\drft^{(L)}\|_{\Hsp_{\knl}}\|\knl(X(s_{n-1}),X(s_{n-1}))\|^p_{\mrm{op}}\Delta^p+\|\s(X(s_{n-1}))\|^p_{\mrm{op}}\|\xi_{n}\|^p\ri).
	\end{align*}
Taking expectation and applying assumptions on $\drft^{(L)}$ and $\dffun$ gives $$\EE_{\drft^{(L)}}\|X(s_n)\|^p \leq \const_{p,0}\lf(1+ \EE_{\drft^{(L)}}\|X(s_{n-1})\|^{p(1\vee\sexp \vee \kexp)}\ri).$$
%\begin{align*}
%	\EE_{\drft^{(L)}}\|X(s_n)\|^p \leq \const_{p,0}\lf(1+ \EE_{\drft^{(L)}}\|X(s_n)\|^{p(\aexp \vee \kexp)}\ri)
%\end{align*}
The assertion follows by iterating the above inequality and using the assumption that $\EE\|X(0)\|^q < \infty$ for any $q>0$.
\end{proof}

The following lemma provides a slight generalization of the convergence of integrals with respect to probability measures under a uniform integrability type condition. The proof is included for completeness.
\begin{lemma}\label{lem:uni-int}
	Let $\{\gmeas^{(L)}\}$ a sequence of probability measures on $\R^{d_0}$ and  $\gmeas^{(L)} \RT \gmeas$ as $L\rt \infty$. Let $ \Lambda: \R^{d_0} \rt [0,\infty)$ be a lower-semicontinuous function such that
	\begin{align}\label{assum-1}
		\cnst_0 \doteq \sup_L \int_{\R^{d_0}} \Lambda(u)\ \gmeas^{(L)}(du) <\infty.
	\end{align}
	Suppose that $\{h^{(L)}: \R^{d_0} \rt \R^{d_1}\} \subset C(\R^{d_0},\R^{d_1})$  is a sequence of continuous functions satisfying the following conditions:
	\begin{enumerate}[label=(\alph*), ref=(\alph*)]
		\item \label{item:uni-cmpt} $h^{(L)} \stackrel{L \rt \infty}\Rt h \in C(\R^{d_0},\R^{d_1})$ uniformly on compact set of $\R^{d_0}$; specifically, for any compact sets $\cmpt \subset \R^{d_0}$, 
		$\dst \sup_{u \in \cmpt}\|h^{(L)}(u) -h(u)\|  \stackrel{L\rt \infty}\Rt  0;$ 
		
		\item \label{item:uni-int} $\sup_{L}\|h^{(L)}(u)\| / \Lambda(u)  \rt 0$, $|h(u)|/\Lambda(u)\rt 0$  as $\|u\| \rt \infty$.
	\end{enumerate}
	Then as $L \rt \infty$, $\dst\int_{\R^{d_0}} h^{(L)}(u) \gmeas^{(L)}(du) \rt \int_{\R^{d_0}} h(u) \gmeas^{(L)}(du)$.
%	\begin{align*}
%		\int_{\R^{d_0}} h^{(L)}(u) \gmeas^{(L)}(du) \rt \int_{\R^{d_0}} h(u) \gmeas^{(L)}(du).
%	\end{align*}	
\end{lemma}

\begin{proof}
	First observe that from \eqref{assum-1}, by  lower semicontinuity of $\Lambda$ and Fatou's lemma %(see \cite[Theorem 1.1]{FKZ14}), we have
	$ \int_{\R^{d_0}} \Lambda(u) \gmeas(du) \leq \cnst_0.$
	Fix an $\vep>0$. Since $\{\gmeas^{(L)}\}$ is tight and \ref{item:uni-int} holds, there exists an $R_0>0$ such that 
	\begin{align*}
		\sup_L\gmeas^{(L)}\{u \in \R^d: \|u\| > R_0\} \leq \vep, \quad \sup_{\|u\| > R_0} \sup_{L}\|h^{(L)}(u)\|\big/\Lambda(u) \leq \vep, \quad \sup_{\|u\| > R_0}|h(u)|/\Lambda(u) \leq \vep.
	\end{align*}		
	Writing $\int_{\R^d} h^{(L)}(u) \gmeas^{(L)}(du) = \int_{\R^d} h(u) \gmeas^{(L)}(du) + \SC{R}_n$, we first show that the term
	$$\SC{R}_n \equiv \int_{\R^d} \lf(h^{(L)}(u) - h(u)\ri) \gmeas^{(L)}(du) \stackrel{L\rt \infty}\Rt 0.$$  To see this, notice that
	\begin{align*}
		|\SC{R}_n| \leq &\ \sup_{\|u\| \leq R_0}\|h^{(L)}(u) - h(u)\| + \int_{\|u\| >R_0} \lf(|h^{(L)}(u)| + |h(u)|\ri) \gmeas^{(L)}(du)\\
		\leq &\ \sup_{\|u\| \leq R_0}\|h^{(L)}(u) - h(u)\|  + \vep \int_{\|x\| >R_0} \Lambda(u) \gmeas^{(L)}(du)\\
		\leq &\ \sup_{\|u\| \leq R_0}\|h^{(L)}(u) - h(u)\| + \vep \cnst_0 \stackrel{L\rt \infty}\rt  \vep \cnst_0.
	\end{align*} 
Since $\vep$ is arbitrary, it follows that $\SC{R}_n \stackrel{n\rt \infty}\rt 0$. The convergence,  $\int_{\R^d} h(u) \gmeas^{(L)}(du) \stackrel{L\rt \infty}\Rt \int_{\R^d} h(u) \gmeas(du)$, follows from standard result on uniform integrability. An easy way to see this is to invoke Skorohod representation theorem to get a probability space $(\tilde \Omega, \tilde{\SC{F}}, \tilde \PP)$, and random variables $U, U^{(L)}, L\geq 1$ on $(\tilde \Omega, \tilde{\SC{F}}, \tilde \PP)$ such that $U^{(L)} \sim \gmeas^{(L)}, \ U \sim \gmeas$ and $U^{(L)}  \rt \tilde U^{(L)} $ $\tilde \PP$-a.s. The conditions on $h$ then imply that 
	the sequence $\{h(U^{(L)})\}$ is uniformly integrable.
\end{proof}	

\begin{lemma}\label{lem:normal}
	  Suppose $V$ and $Z$ are respectively $\R^{d_0}$ and $\R^{d_1}$-valued random variables satisfying the following
	\begin{align*}
		V|Z=z \sim \No_{d_0}(\omat z+ \mfk{g}, Q), \quad Z \sim \No_{d_1}(\mfk{f}, P),
	\end{align*}
	where $\omat \in \R^{d_0\times d_1}, Q \in \R^{d_0\times d_0}, P \in \R^{d_1\times d_1}, \mfk{g} \in \R^{d_0}, \mfk{f} \in R^{d_1}$. Assume that $P$ and $Q$ are non-singular, and the parameters $\mfk{g}$ and $Q$ do not depend on $z$. Then
	$Z|V=v \sim \No_{d_1}(\mfk{m}, S)$, where
	\begin{align*}
		S =&\ (\omat^\top Q^{-1}\omat +P)^{-1} = P- P\omat^\top(Q+\omat P \omat^\top)^{-1}\omat P\\
		\mfk{m} =&\ \mfk{f}+S\omat^\top Q^{-1}(v-\mfk{g}-\omat \mfk{f})=\mfk{f}+P\omat^\top(Q+\omat P\omat^\top)^{-1}(v-\mfk{g}-\omat \mfk{f}).
	\end{align*}
	
\end{lemma}

\np
{\bf Linear SDE:} 
Consider the possibly time-inhomogeneous linear SDE:
\begin{align*}
	dX(t) = \big(B(t)X(t)+v(t)\big)dt + \dffun(t) dW(t), \quad X(0) \sim \No (x_0, S_0).
\end{align*}
Then $X(t) \sim \No(\mfk{m}(t), S(t))$ where $\mfk{m}$ and $S$ satisfy the ODE system
\begin{align*}
	\f{d\mfk{m}(t)}{dt} =&\ B(t) \mfk{m}(t) + v(t), \quad \mfk{m}(0) = x_0\\
	\f{dS(t)}{dt}=&\ B(t) S(t)+S(t)B^\top(t)+\s\s^\top(t), \quad S(0) = S_0.
\end{align*}
The last equation is an example of differential Sylvester equation --- a special case of differential Lyapunov equation \cite{BeBeHe19}. In general, we need a numerical ODE solver to solve this. But in the case of time homogeneous SDE, that is, when $B, v$ and $\s$ do not depend on $t$, and when $B$ is invertible, we can get almost a closed form expression of the solution given by
\begin{align}\label{eq:lsde-mv}
	\begin{aligned}
		\mfk{m}(t) =&\ -B^{-1}v+e^{Bt}(x_0+B^{-1}v) = e^{Bt}x_0+ \lf(e^{Bt}-I\ri)B^{-1}v\\
		%\mfk{m}(t) =&\ e^{At}x_0+ \lf(e^{At}-I\ri)A^{-1}b\\
		S(t)=&\ e^{Bt}S_0e^{B^\top t}+\int_0^t e^{B(t-s)}\s\s^\top e^{B^\top(t-s)}dt.
	\end{aligned}
\end{align}

\bibliographystyle{plainnat}
\bibliography{Ref-ML}

\end{document}